\documentclass{article}

\PassOptionsToPackage{round}{natbib}

\usepackage{natbib}
\usepackage[preprint]{neurips_2022}
\usepackage{amstext,amsmath,amssymb,amsthm}
\usepackage{algorithmicx,algorithm,algpseudocode}
\newtheorem{proposition}{Proposition}[section]
\newtheorem{lemma}{Lemma}[section]
\newtheorem{definition}{Definition}

\usepackage[utf8]{inputenc} 
\usepackage[T1]{fontenc}    
\usepackage{hyperref}       
\usepackage{url}            
\usepackage{booktabs}       
\usepackage{amsfonts}       
\usepackage{amsmath,amsfonts,amssymb,amsthm}
\usepackage{cases}
\usepackage{nicefrac}       
\usepackage{microtype}      
\usepackage{xcolor}         

\usepackage{wrapfig}
\usepackage{subfig}
\usepackage{graphicx}       

\hypersetup{
	colorlinks=true,%
	citecolor={blue!50!black},
	linkcolor={red!50!black},
	urlcolor={green!50!black}
}

\title{Learning to Share in Multi-Agent RL}

\author{
	   Yuxuan Yi\textsuperscript{$\dagger$} \quad Ge Li\textsuperscript{$\dagger$} \quad Yaowei Wang\textsuperscript{$\ddagger$} \quad Zongqing Lu\textsuperscript{$\dagger$}\\ 
	   \textsuperscript{$\dagger$}Peking University \quad \textsuperscript{$\ddagger$}Peng Cheng Lab \\
	   \texttt{\{touma,geli,zongqing.lu\}@pku.edu.cn} \quad \texttt{wangyw@pcl.ac.cn}
}

\begin{document}

	\maketitle
	
	\begin{abstract}
		In this paper, we study the problem of networked multi-agent reinforcement learning (MARL), where a number of agents are deployed as a partially connected network and each interacts only with nearby agents. Networked MARL requires all agents to make decisions in a decentralized manner to optimize a global objective with restricted communication between neighbors over the network. Inspired by the fact that \textit{sharing} plays a key role in human's learning of cooperation, we propose LToS, a hierarchically decentralized MARL framework that enables agents to learn to dynamically share reward with neighbors so as to encourage agents to cooperate on the global objective
		through collectives.
		For each agent, the high-level policy learns how to share reward with neighbors to decompose the global objective, while the low-level policy learns to optimize the local objective induced by the high-level policies in the neighborhood. The two policies form a bi-level optimization and learn alternately. We empirically demonstrate that LToS outperforms existing methods in both social dilemma and networked MARL scenarios
		across scales.
	\end{abstract}
	
	\section{Introduction}
	
	In fully cooperative multi-agent reinforcement learning (MARL), 
	there are multiple agents interacting with the environment via their joint action to cooperatively optimize an objective. Many methods of centralized training and decentralized execution (CTDE) have been proposed for cooperative MARL, such as COMA \citep{foerster2018counterfactual}, QMIX \citep{rashid2018qmix}, QPLEX \citep{wang2021qplex}, and FOP \citep{zhang2021fop}. However, these methods suffer from the overgeneralization issue:
	employed value functions cannot estimate well because agents sometimes choose uncoordinated actions, and thus the optimal policy cannot be learned
	\citep{castellini2019representational}. 
	Moreover, they may not easily scale up with the number of agents due to centralized learning \citep{qu2020intention}.
	
	In many MARL applications, there are a large number of agents that are deployed as a partially connected network and collaboratively make decisions to optimize the globally averaged return, such as communication networks \citep{kim2018learning} and traffic signal control \citep{wei2019colight}. 
	To deal with such scenarios, networked MARL is formulated to decompose the dependency among all agents into dependencies between only neighbors. To avoid decision-making with insufficient information, agents are permitted to exchange messages with neighbors over the network. In such settings, it is feasible for agents to learn to make decisions in a decentralized way \citep{zhang2018fully,qu2019scalable2}. However, there are still difficulties of dependency if anyone attempts to make decisions independently, \textit{e.g., prisoner's dilemma} and \textit{tragedy of the commons} \citep{perolat2017multi}. Existing methods tackle these problems by consensus update of value function \citep{zhang2018fully}, credit assignment \citep{wang2020shapley}, or reward shaping \citep{chu2020multi}. However, these methods rely on either access to global state and joint action \citep{zhang2018fully} or hand-crafted reward functions \citep{wang2020shapley,chu2020multi}. 
	
	Inspired by the fact that \textit{sharing} plays a key role in human's learning of cooperation \citep{eisenberg1989roots}, we propose \textit{\textbf{learning to share}} (\textbf{LToS}), a hierarchically decentralized learning framework for networked MARL. LToS enables agents to learn to dynamically share reward with neighbors so as to collaboratively optimize the global objective. The high-level policies decompose the global objective into local ones by determining how to share their rewards, while the low-level policies optimize local objectives induced by the high-level policies. LToS learns in a decentralized manner, and we prove that the high-level policies are a mean-field approximation of the joint high-level policy. Moreover, the high-level and low-level policies form a bi-level optimization and alternately learn to optimize the global objective. 
	
	LToS is a general hierarchical framework for networked MARL and can be easily realized by diverse combinations of RL algorithms. We currently implement LToS by DDPG \citep{lillicrap2016continuous} as the high-level policy and DGN \citep{jiang2020graph} as the low-level policy. We empirically demonstrate that LToS outperforms existing methods for networked MARL in both social dilemma and networked MARL scenario.

	\section{Related Work}
	There are many recent studies for collaborative MARL. Most adopt CTDE
	\citep{foerster2018counterfactual,rashid2018qmix,wang2021qplex,zhang2021fop,su2022dmac}. Many of them are constructed on the basis of factorizing the joint Q-function \citep{rashid2018qmix,wang2021qplex}. However, these factorized methods suffer from the overgeneralization issue \citep{castellini2019representational}. Other studies focus more on \textit{decentralized training}, to which our work is more closely related, as summarized as follows.
	
	\textbf{Networked MARL.} 
	\citet{zhang2018fully} and \citet{qu2019value} proposed consensus update of local value functions, where each agent keeps a local copy of the global value function but is assumed to have global information. \citet{qu2020intention} proposed intention propagation between agents, where each agent updates its policy based on intentions shared by other agents, but the policy may converge slowly due to propagated intentions over the network. \citet{qu2019scalable2} and \citet{lin2020multiagent} investigated the exponential decay property, \textit{i.e.}, the impact of agents on each other decays exponentially in their graph distance, while \citet{chu2020multi} introduced a spatial discount factor to capture the influence between agents, which remains hand-tuned. However, none of these studies provide an explicit mechanism to solve social dilemma in networked MARL.

	\textbf{Reward Design.} 
	\citet{hostallero2020inducing} aimed at maximizing social welfare, but they simply used temporal difference error for reward shaping. As temporal difference error in deep RL hardly converges to zero, it still biases the optimization objective. \citet{mguni2019coordinating} added an extra part to the original reward as non-potential based reward shaping and used Bayesian optimization to induce the convergence to a desirable equilibrium between agents. However, the extra part remains fixed during an episode, which makes it less capable of dealing with dynamic environments. Moreover, the reward shaping alters the original optimization problem. \citet{hughes2018inequity} proposed the inequity aversion model to balance agents' selfish desire and social fairness. \citet{wang2020shapley} considered learning the Shapley value as the credit assignment. However, these methods still rely on hand-crafted reward designs. \citet{lupu2020gifting} added \textit{gifting} as an extra action into the original MDP to modify the MDP and objective. \citet{yang2020learning} proposed that each agent learns an incentive function and optimizes the policy in terms of both reward and incentives given by other agents. Obviously, both methods alter the original objective of optimization. 
	
	Unlike existing work, LToS enables agents to learn to dynamically share reward with other agents \textit{without the bias of the optimization objective} such that they can collaboratively optimize the global objective in networked MARL.

	\section{Background}
	
	\subsection{Networked Multi-Agent Reinforcement Learning}
	Assume $N$ agents interact with an environment. Let $\mathcal{V} = \{1, 2, \cdots, N\}$ be the set of agents. The multi-agent system is modeled as an undirected graph $\mathcal{G}\mathcal{(V, E})$, where each agent $i$ serves as vertex $i$ and $\mathcal{E} \mathcal{\subseteq V \times V}$ is the set of all edges. Two agents $i,j \in \mathcal{V}$ can communicate with each other if and only if $e_{ij} = (i,j) \in \mathcal{E}$. We denote \textit{agent $i$ and its all neighbors in the graph together as a set $\mathcal{N}_{i}$}. The state of the environment $s \in \mathcal{S}$ transitions upon joint action $\boldsymbol{a} \in \mathcal{A}$ according to transition probability $\mathcal{P}_a: \mathcal{S \times A \times S} \rightarrow [0,1]$, where joint action set $\mathcal{A} = \times_{i \in \mathcal{V}} \mathcal{A}_i$. Each agent $i$ has a policy $\pi_i \in \Pi_i: \mathcal{S} \times \mathcal{A}_i \rightarrow [0,1]$, and we denote the joint policy of all agents as $\boldsymbol{\pi} \in \Pi = \times_{i \in \mathcal{V}} \Pi_i$ \citep{zhang2018fully}. For networked MARL, a common and realistic assumption is that the reward of each agent $i$ just depends on its action and the actions of its neighbors \citep{qu2020intention}, \textit{i.e.}, $r_i(s,\boldsymbol{a})=r_i(s,a_{\mathcal{N}_{i}})$. Moreover, each agent $i$ may only obtain partial observation $o_i \in \mathcal{O}_i$, but can approximate the state by the observations of $\mathcal{N}_{i}$ \citep{jiang2020graph} or the observation history \citep{chu2020multi}, which are all denoted by $o_i$ for \textit{simplicity}.
	The \textit{global objective} is to maximize the sum of cumulative rewards of all agents , \textit{i.e.}, $\sum_{t=0}^{\infty} \sum_{i=1}^{N}\gamma^t r_i^t$.
	
	\subsection{Markov Game}
	
	In such a setting, each agent could individually maximizes its own expected return, which is known as Markov game. This may lead to stable outcome or Nash equilibrium, which however is usually sub-optimal
	in terms of the global objective. Given $\boldsymbol{\pi}$, the value function of agent $i$ is given by
	\begin{equation}
		\label{eq:v_i_simple}
		v_i^{\boldsymbol{\pi}}(s) = \sum_{\boldsymbol{a}} \boldsymbol{\pi}(\boldsymbol{a}|s) \sum_{s^\prime} p_a(s^\prime|s,\boldsymbol{a}) [r_i + \gamma v_i^{\boldsymbol{\pi}}(s^\prime)],
	\end{equation}
	where $p_a \in \mathcal{P}_a$ describes the state transitions. A Nash equilibrium is defined as \citep{mguni2019coordinating}
	\begin{equation*}
		\label{M-NE}
		v_i^{(\pi_i,\pi_{-i})}(s) \geq v_i^{(\pi_i^\prime,\pi_{-i})}(s), \quad \forall \pi_i^\prime \in \Pi_i, \forall s \in \mathcal{S}, \forall i \in \mathcal{V},
	\end{equation*}
	where $\pi_{-i} = \times_{j \in \mathcal{V}\backslash \{i\}} \pi_j$.

	\section{Method}
	
	LToS is a decentralized hierarchy. At each agent, the high-level policy determines the weights of reward sharing based on low-level policies while the low-level policy directly interacts with the environment to optimize the local objective induced by the high-level policies. Therefore, they form a bi-level optimization and alternately learn towards the global objective. 
	
	\subsection{Reward Sharing}
	
	The intuition of reward sharing is that if agents share their rewards with others, each agent has to consider the consequence of its actions on others, and thus it promotes cooperation. In networked MARL, as the reward of an agent is assumed to depend on the actions of neighbors, we allow reward sharing only between neighboring agents. This is because the change of actions of neighbors directly affects the reward while the agents outside the neighborhood can only affect the return of the agent indirectly by the change of state distribution. Moreover, this also fits the setting of networked MARL with restricted communication between neighbors.
	
	For the graph of $\mathcal{V}$, we additionally define a set of directed edges, $\mathcal{D}$, constructed from $\mathcal{E}$. Specifically, we add a loop $d_{ii} \in \mathcal{D}$ for each agent $i$ and split each undirected edge $e_{ij} \in \mathcal{E}$ into two directed edges: $d_{ij} = (i, j)$ and $d_{ji} = (j, i) \in \mathcal{D}$. Each agent $i$ determines a weight $w_{ij} \in [0,1]$ for each directed edge $d_{ij}, \forall j \in \mathcal{N}_{i}$, subject to the constraint $\sum_{j\in \mathcal{N}_{i}} w_{ij} = 1$, so that $w_{ij}$ proportion of agent $i$'s environment reward $r_i$ will be shared to agent $j$. Let $\boldsymbol{w} \in \mathcal{W} = \times_{d_{ij} \in \mathcal{D}} w_{ij}$ be the weights of the graph. Therefore, the shaped reward after sharing for each agent $i$ is defined as 
	\begin{equation}
		\label{eq:w}
		r_{i}^{\boldsymbol{w}} = \sum_{j \in \mathcal{N}_{i}} w_{ji}r_j.
	\end{equation}
	
	\subsection{Hierarchy}
	
	Assume there is a joint high-level policy $\boldsymbol{\phi} \in \Phi: \mathcal{S \times W} \rightarrow [0,1]$ to determine $\boldsymbol{w}$. Given $\boldsymbol{\phi}$ and $\boldsymbol{w}$, we can define the value function of $\boldsymbol{\pi}$ at each agent $i$ based on (\ref{eq:v_i_simple}) as
	\begin{gather}
		\begin{split}
			\label{v_i_phi}
			v_{i}^{\boldsymbol{\pi}}(s;\boldsymbol{\phi})
			= \sum_{\boldsymbol{w}} \boldsymbol{\phi}(\boldsymbol{w}|s) \sum_{\boldsymbol{a}}  \boldsymbol{\pi}(\boldsymbol{a}|s,\boldsymbol{w}) 
			\sum_{s^\prime} p_a(s^\prime|s,\boldsymbol{a}) [r_{i}^{\boldsymbol{w}} + \gamma v_{i}^{\boldsymbol{\pi}}(s^\prime;\boldsymbol{\phi})],
		\end{split}
		\\
		\begin{split}
			\label{v_i_w}
			\nonumber
			v_{i}^{\boldsymbol{\pi}}(s;\boldsymbol{w}, \boldsymbol{\phi})
			= \sum_{\boldsymbol{a}}  \boldsymbol{\pi}(\boldsymbol{a}|s,\boldsymbol{w})
			\sum_{s^\prime} p_a(s^\prime|s,\boldsymbol{a}) [r_{i}^{\boldsymbol{w}} + \gamma v_{i}^{\boldsymbol{\pi}}(s^\prime; \boldsymbol{\phi})].
		\end{split}
	\end{gather}
	
	It is noteworthy that $\boldsymbol{w}$ is a multidimensional action for an allocation scheme rather than a probability distribution. In our derivation, we
	express $\boldsymbol{w}$ as a discrete action for simplicity. It also holds for continuous action as long as we change all the summations to integrals.
	Let $V_{\mathcal{V}}^{\boldsymbol{\phi}}(s;\boldsymbol{\pi}) \doteq \sum_{i \in \mathcal{V}} v_{i}^{\boldsymbol{\pi}}(s; \boldsymbol{\phi})$ and $Q_{\mathcal{V}}^{\boldsymbol{\phi}}(s, \boldsymbol{w};\boldsymbol{\pi}) \doteq \sum_{i \in \mathcal{V}} v_{i}^{\boldsymbol{\pi}}(s; \boldsymbol{w}, \boldsymbol{\phi})$.
	
	\begin{proposition}
		\label{prop:1}
		Given $\boldsymbol{\pi}$, $V_{\mathcal{V}}^{\boldsymbol{\phi}}(s;\boldsymbol{\pi})$ and $Q_{\mathcal{V}}^{\boldsymbol{\phi}}(s, \boldsymbol{w};\boldsymbol{\pi})$ are respectively the value function and action-value function of $\boldsymbol{\phi}$.
	\end{proposition}
	\begin{proof}
		The proof is deferred to Appendix \ref{app:proofs}. 
	\end{proof}
	
	Proposition \ref{prop:1} implies that $\boldsymbol{\phi}$ directly optimizes the global objective by generating $\boldsymbol{w}$, given $\boldsymbol{\pi}$. Unlike existing hierarchical RL methods, we can directly construct the value function and action-value function of $\boldsymbol{\phi}$ based on the value function of $\boldsymbol{\pi}$ at each agent. 
	
	As $\boldsymbol{\phi}$ optimizes the global objective given $\boldsymbol{\pi}$ while $\pi_i$ optimizes the shaped reward individually at each agent given $\boldsymbol{\phi}$ (assuming $\boldsymbol{\pi}$ convergent to Nash equilibrium or stable outcome, denoted as $\lim$), they form a bi-level optimization. Let $J_{\boldsymbol{\phi}}(\boldsymbol{\pi})$ and
	$J_{\boldsymbol{\pi}}(\boldsymbol{\phi})$ denote the objectives of $\boldsymbol{\phi}$ and $\boldsymbol{\pi}$ respectively. 
	The bi-level optimization can be formulated as follows,
	\begin{equation}
		\begin{aligned}
			\label{eq:bilevel-c}
			\max_{\boldsymbol{\phi}} & \quad J_{\boldsymbol{\phi}}(\boldsymbol{\pi}^*(\boldsymbol{\phi})) \\
			s.t. & \quad \boldsymbol{\pi}^*(\boldsymbol{\phi}) = \arg \lim_{\boldsymbol{\pi}} J_{\boldsymbol{\pi}}(\boldsymbol{\phi}).
		\end{aligned}
	\end{equation}
	
	\subsection{Decentralized Learning}
	
	We start from collective learning to achieve global optimization of average reward.
	So far, the joint high-level policy is still in a centralized form. 
	Note that the scenario needs a decentralized method and each agent has its own reward.
	Now we turn to learning the joint high-level policy in a decentralized way.
	Let $w_{i}^{\text{out}} \doteq \{w_{ij} | j \in \mathcal{N}_{i}\}$ and $w_{i}^{\text{in}} \doteq \{w_{ji} | j \in \mathcal{N}_{i}\}$.
	The following proposition proves each agent’s independence of each other on the high level. 
	
	\begin{proposition}
		\label{prop:3}
		The joint high level policy $\boldsymbol{\phi}$ can be learned in a decentralized manner, and the decentralized high-level policies of all agents form a mean-field approximation of $\boldsymbol{\phi}$.
	\end{proposition}
	\begin{proof}
		The proof is deferred to Appendix \ref{app:proofs}. 
	\end{proof}
	
	Proposition~\ref{prop:1} and \ref{prop:3} indicate that for each agent $i$, the low-level policy simply learns a local $\pi_i(a_i|s,w_{i}^{\text{in}})$
	to optimize the cumulative reward of $r_{i}^{\boldsymbol{w}}$, since $r_{i}^{\boldsymbol{w}}$ is fully determined by $w_{i}^{\text{in}}$ according to (\ref{eq:w}) and denoted as $r_{i}^{w}$ from now on. And the high-level policy $\phi_i$ just needs to locally determine $w_{i}^{\text{out}}$
	to optimize the cumulative reward of $r_{\mathcal{V}}^{\boldsymbol{\phi}}$. 
	
	

	Therefore, for decentralized learning, (\ref{eq:bilevel-c}) can be decomposed locally for each agent $i$ as
	\begin{equation}
		\begin{aligned} 
			\label{eq:bilevel-d}
			\max_{\phi_i} \quad & J_{\phi_i}(\phi_{-i}, \pi_1^*(\boldsymbol{\phi}),\cdots, \pi_N^*(\boldsymbol{\phi}))\\
			s.t. \quad & \pi_i^*(\boldsymbol{\phi}) = \arg\max_{\pi_i} J_{\pi_i}(\pi_{-i},\phi_1(\boldsymbol{\pi}),\cdot\cdot,\phi_N(\boldsymbol{\pi})).
		\end{aligned}
	\end{equation}
	Now we use $\max$ instead of $\lim$ because local policies can be compared and improved in a decentralized manner in a Markov game. 
	We abuse the notation and let $\phi$ and $\pi$ also denote their parameterizations respectively. To solve the optimization, we have
	\begin{equation}
		\label{approximation}
		\begin{aligned} 
			\nabla_{\phi_i} J_{\phi_i}(\phi_{-i}, &\, \pi_1^*(\boldsymbol{\phi}),\cdots,\pi_N^*(\boldsymbol{\phi})) \\
			&\approx \nabla_{\phi_i} J_{\phi_i}(\phi_{-i}, \pi_1 + \alpha \nabla_{\pi_1} J_{\pi_1}(\boldsymbol{\phi}),\cdots,\pi_N + \alpha \nabla_{\pi_N} J_{\pi_N}(\boldsymbol{\phi})),
		\end{aligned}
	\end{equation}
	where $\alpha$ is the learning rate for the low-level policy. Let $\pi_i^\prime$ denote $\pi_i + \alpha \nabla_{\pi_i} J_{\pi_i}(\boldsymbol{\phi})$, we have
	\begin{equation}
		\nonumber
		\begin{aligned} 
			\nabla_{\phi_i} J_{\phi_i}&(\phi_{-i}, \pi_1^*(\boldsymbol{\phi}),\cdots,\pi_N^*(\boldsymbol{\phi})) \\
			&\approx \nabla_{\phi_i} J_{\phi_i}(\phi_{-i}, \pi_1^\prime, \cdots, \pi_N^\prime) + \alpha \sum_{j=1}^N \nabla_{\phi_i,\pi_j}^2 J_{\pi_j}(\boldsymbol{\phi}) \nabla_{\pi_j^\prime} J_{\phi_i}(\phi_{-i}, \pi_1^\prime, \cdots, \pi_N^\prime).
		\end{aligned}
	\end{equation}
	
	The second-order derivative is neglected due to high computational complexity, without incurring significant performance drop such as in meta-learning \citep{finn2017modelagnostic}
	and neural architecture search \citep{liu2019darts}. 
	Differently, our low-level policy requires more than one gradient step until convergence.
	Similarly, we have 
	\begin{equation}
		\nonumber
		\begin{aligned} 
			\nabla_{\pi_i} J_{\pi_i} (\pi_{-i}, &\, \phi_1^*(\boldsymbol{\pi}),\cdots,\phi_N^*(\boldsymbol{\pi})) \\
			& \approx \nabla_{\pi_i} J_{\pi_i}(\pi_{-i}, \phi_1 + \beta \nabla_{\phi_1}J_{\phi_{1}}(\boldsymbol{\pi}), \cdots, \phi_N + \beta \nabla_{\phi_N} J_{\phi_N}(\boldsymbol{\pi})), 
		\end{aligned}
	\end{equation}
	where $\beta$ is the learning rate of the high-level policy. Therefore, we can solve the bi-level optimization (\ref{eq:bilevel-c}) by the first-order approximations in a decentralized way. For each agent $i$, $\phi_i$ and $\pi_i$ are alternately updated. 
	
	\begin{wrapfigure}{r}{0.4\textwidth}
		\setlength{\abovecaptionskip}{5pt}
		\vspace{-0.35cm}
		\centering
		\includegraphics[width=0.4\textwidth]{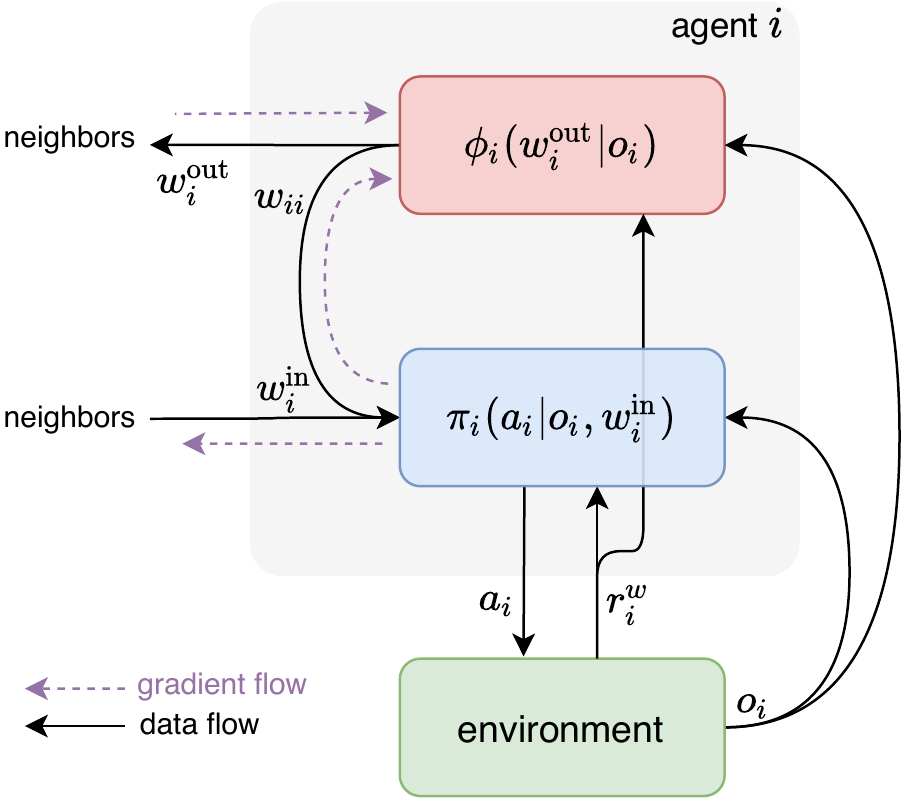}
		\caption{LToS}
		\label{fig:arch}
		\vspace{-0.3cm}
	\end{wrapfigure}
	
	In distributed learning, as each agent $i$ usually does not have access to state, we further approximate $\phi_i(w_{i}^{\text{out}}|s)$ and $\pi_i(a_i|s,w_{i}^{\text{in}})$ by $\phi_i(w_{i}^{\text{out}}|o_i)$ and $\pi_i(a_i|o_i,w_{i}^{\text{in}})$, respectively. Moreover, in network MARL as each agent $i$ is closely related to neighboring agents, (\ref{eq:bilevel-d}) can be further seen as $\pi_i$ maximizes the cumulative discounted reward of $r_{i}^{w}$ given $\phi_{\mathcal{N}_{i}}$, where $\phi_{\mathcal{N}_{i}}=\times_{j \in \mathcal{N}_{i}} \phi_j$, 
	and $\phi_i$ equivalently optimizes the global objective
	given $\pi_{\mathcal{N}_{i}}$
	, where $\pi_{\mathcal{N}_{i}}=\times_{j \in \mathcal{N}_{i}} \pi_j$. 
	During training, $\pi_{\mathcal{N}_{i}}$ and $\phi_{\mathcal{N}_{i}}$ are implicitly considered by interactions of $w_{i}^{\text{out}}$ and $w_{i}^{\text{in}}$ respectively. The architecture of LToS is illustrated in Figure~\ref{fig:arch}. At each timestep, the high-level policy of each agent $i$ makes a decision of action $w_{i}^{\text{out}}$ as the weights of reward sharing based on the observation. Then, the low-level policy takes the observation and $w_{i}^{\text{in}}$ as an input and outputs the action. Agent $i$ obtains the shaped reward according to $w_{i}^{\text{in}}$ for both the high-level and low-level policies. The gradients are backpropagated along purple dotted lines.

	Further, from Proposition \ref{prop:1}, we have:
		$
		q_{i}^{\phi_i}(s,w_{i}^{\text{out}};\pi_{\mathcal{N}_{i}}) = v_{i}^{\pi_i}(s;w_{i}^{\text{in}},\phi_{\mathcal{N}_{i}})
		$,
	where $q_{i}^{\phi_i}$ is the action-value function of $\phi_i$ given $\pi_{\mathcal{N}_{i}}$, $v_{i}^{\pi_i}$ is the value function of $\pi_i$ given $\phi_{\mathcal{N}_{i}}$ and conditioned on $w_i^\text{in}$. As aforementioned, we approximately have $q_{i}^{\phi_i}(o_i,w_{i}^{\text{out}}) = v_{i}^{\pi_i}(o_i;w_{i}^{\text{in}})$. We can see that the action-value function of $\phi_i$ is equivalent to the value function of $\pi_i$. That said, we can use a single network to approximate these two functions simultaneously. For a deterministic low-level policy, the high-level and low-level policies can share 
	the
	same action-value function. In the current instantiation of LToS, we use DGN \citep{jiang2020graph} (Q-learning) for the low-level policy and DDPG \citep{lillicrap2016continuous} for the high-level policy. Thus, the Q-network of DGN also serves as the critic of DDPG, and the gradient of $w_{i}^{\text{in}}$ is calculated based on the maximum Q-value of $a_i$. 
	
	For completeness, Algorithm \ref{alg:ltos} (see Appendix \ref{algorithm}) gives the training procedure of LToS based on DDPG and DGN. More discussions about training LToS are also available in Appendix~\ref{discussion}.

	\section{Experiments}

	For the experiments, we adopt three scenarios \textit{prisoner}, \textit{jungle}, and \textit{traffic} depicted in Figure~\ref{fig:scenarios}, where \textit{prisoner} and \textit{jungle} \citep{jiang2020graph} are grid games about social dilemma that easily measures agents' cooperation, while \textit{traffic} is a realistic scenario of networked MARL. We obey the principle of networked MARL that only allows communication in neighborhood as \citet{zhang2018fully} and \citet{chu2020multi}.
	
	To illustrate the reward sharing scheme each agent learned, we use a simple indicator: \textit{selfishness}, the reward proportion that an agent chooses to keep for itself. For ablation, we keep the sharing weights fixed for each agent, named \textit{fixed} LToS. Throughout the experiments, we additionally compare with the baselines including DQN and DGN, where DGN also serves the ablation of LToS without reward sharing as DGN is the low-level policy of LToS. To maximize the global return directly by centralized learning, we use QMIX \citep{rashid2018qmix} as a baseline throughout the three scenarios and two other ones in \textit{prisoner}. Moreover, as LToS aims to bring a harmonious cooperation by reward sharing in networked MARL, we compared LToS to three methods for networked MARL, \textit{i.e.}, ConseNet \citep{zhang2018fully}, NeurComm \citep{chu2020multi} and Intention Propagation (abbreviated as IP) \citep{qu2020intention}, and LIO \citep{yang2020learning} for incentivized learning, all of which use recurrent neural network (RNN) or graph neural network (GNN) for the partially observable environment.
	More details of hyperparameters are available in Appendix~\ref{sec:hyper}.
	
	\begin{figure*}[!t]
		\centering
		\setlength{\abovecaptionskip}{5pt}
		\begin{minipage}{1\textwidth}
			\setlength{\abovecaptionskip}{5pt}
			\centering
			\subfloat[][\textit{prisoner}]{
				\includegraphics[width=0.7\textwidth]{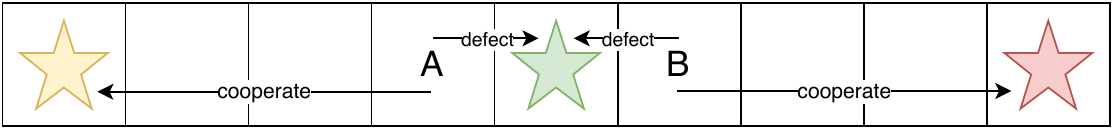}
				\label{fig:prisoner-scenario}}
		\end{minipage}
		\vskip 0.2cm
		\begin{minipage}{0.3\textwidth}
			\setlength{\abovecaptionskip}{5pt}
			\centering
			\subfloat[][\textit{jungle}]{
				\includegraphics[width=1\textwidth]{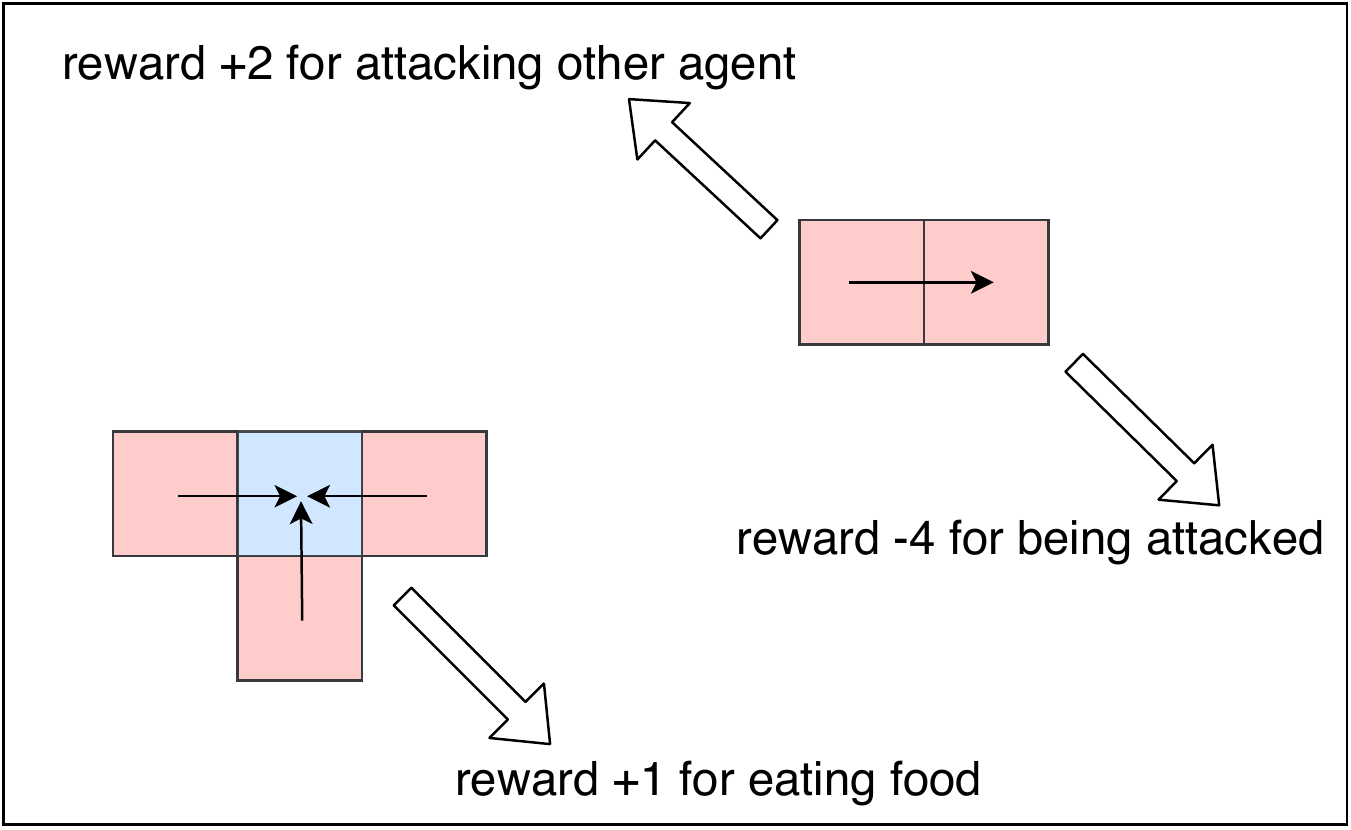}
				\label{fig:jungle-scenario}}
		\end{minipage} \qquad
		\begin{minipage}{0.35\textwidth}
			\setlength{\abovecaptionskip}{5pt}
			\centering
			\subfloat[][\textit{traffic}]{
				\includegraphics[width=1\textwidth]{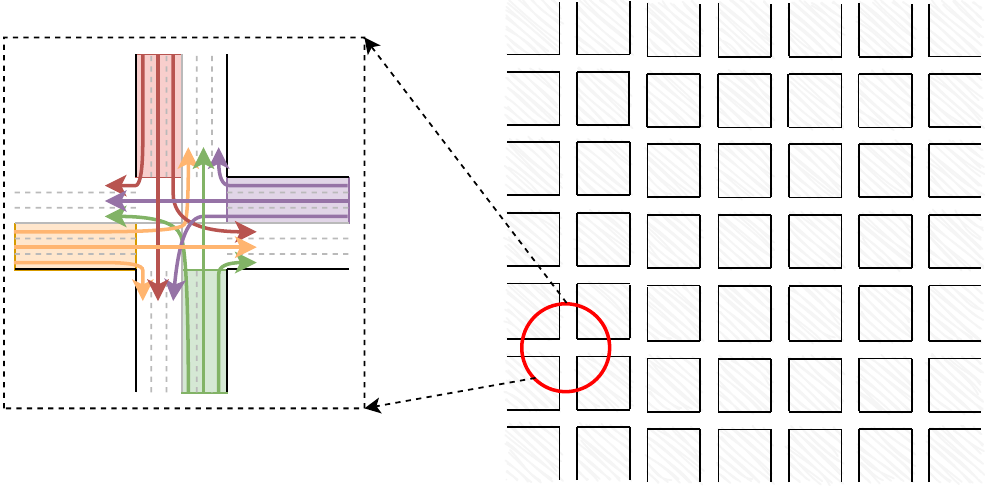}
				\label{fig:traffic-scenario}}
		\end{minipage}
		\caption{Three experimental scenarios: (a) \textit{prisoner}, (b) \textit{jungle}, and (c) \textit{traffic}.}
		\label{fig:scenarios}
		\vspace*{-3mm}
	\end{figure*}
	
	\subsection{Prisoner}
	
	We use \textit{prisoner}, a grid game version of the well-known matrix game \textit{prisoner's dilemma} from \citet{sodomka2013coco} to empirically demonstrate that LToS is able to learn cooperative policies to achieve the global optimum (\textit{i.e.}, maximize globally averaged return). As illustrated in Figure~\ref{fig:prisoner-scenario}, there are two agents $A$ and $B$ that respectively start on two sides of the middle of a grid corridor with \textit{full} observation. At each timestep, each agent chooses an action \textit{left} or \textit{right} and moves to the corresponding adjacent grid, and every action incurs a cost $-0.01$. There are three goals, two goals at both ends and one in the middle. The agent gets a reward $+1$ for reaching the goal. The game ends once some agent reaches a goal or two agents reach different goals simultaneously. This game resembles prisoner's dilemma: going for the middle goal (``defect") will bring more rewards than the farther one on its side (``cooperate"), but if two agents both adopt that, a collision occurs and only one of the agents wins the goal with equal probability. On the contrary, both agents obtain a higher return if they both ``cooperate", though it takes more steps. The highest possible return is $1$.
	
	Figure~\ref{fig:prisoner-curve} illustrates the learning curves of all the methods in terms of average return. Note that for all three scenarios, we present the average of 5 training runs with different random seeds by solid lines and the min/max value by shadowed areas. As a result of self-interest optimization, DQN converges to the ``defect/defect" Nash equilibrium where each agent receives an expected reward about $0.5$. So does DGN since it only aims to take advantage of its neighbors' observations while \textit{prisoner} is a fully observable environment already. Given a hand-tuned reward shaping factor to direct agents to maximize average return, NeurComm and \textit{fixed} LToS agents are able to cooperate eventually. So are ConseNet and QMIX. However, they converge slowly. In contrast, IP agents learn at a slower pace and its performance is only a little higher than $0.5$. LIO agents cooperate soon enough at the beginning, but they cannot form steady cooperation and perhaps need longer training to get rid of such instability. 
	
	JointDQN, Coco-Q \citep{sodomka2013coco}, and LToS perform similarly and outperform other methods. JointDQN is one centralized DQN that takes control of joint actions of both agent $A$ and $B$, and thus should be able to achieve the best performance but still takes time to converge even in such a simple two-agent scenario. As a modified tabular Q-learning method, Coco-Q introduces the coco value \citep{kalai2010cooperation} as a substitute for the expected return in the Bellman equation and regards the difference as transferred reward. However, it is specifically designed for some games, and it is hard to be extended beyond two-player games. LToS can learn the reward sharing scheme where one agent at first gives all the reward to the other so that both of them are prevented from ``defect", and thus achieve the best average return quickly, as observed in the experiment. By \textit{prisoner}, we verify that LToS can escape from local optimum by learning to share reward.

	\begin{figure*}[!t]
		\centering
		\setlength{\abovecaptionskip}{5pt}
		\begin{minipage}[t]{0.318\textwidth}
			\setlength{\abovecaptionskip}{5pt}
			\centering
			\subfloat[][\textit{prisoner}]{
				\includegraphics[width=1\textwidth]{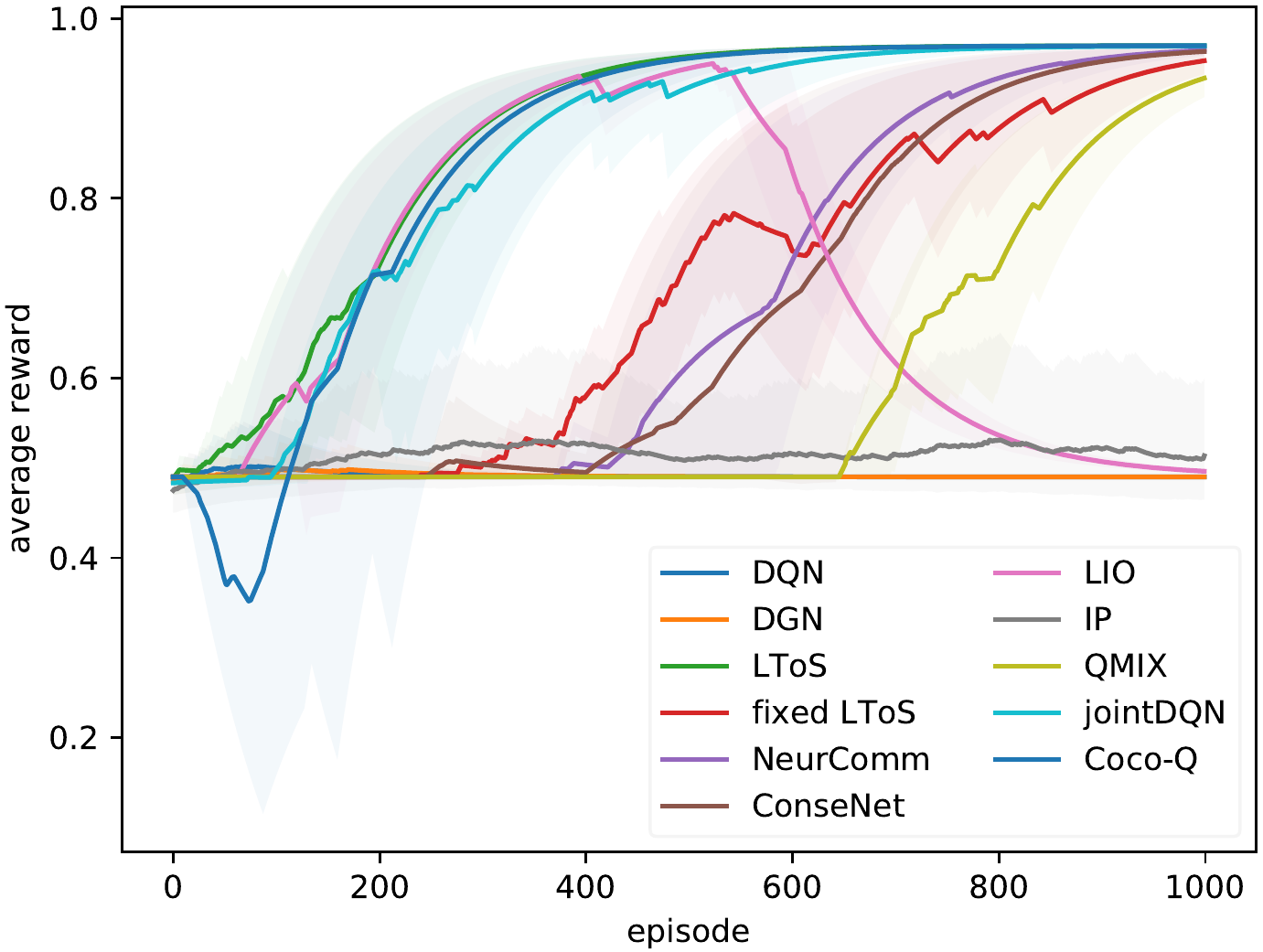}
				\label{fig:prisoner-curve}
			}
		\end{minipage}
		\hskip 0.1cm
		\begin{minipage}[t]{0.323\textwidth}
			\setlength{\abovecaptionskip}{5pt}
			\centering
			\subfloat[][\textit{jungle}]{
				\includegraphics[width=1\textwidth]{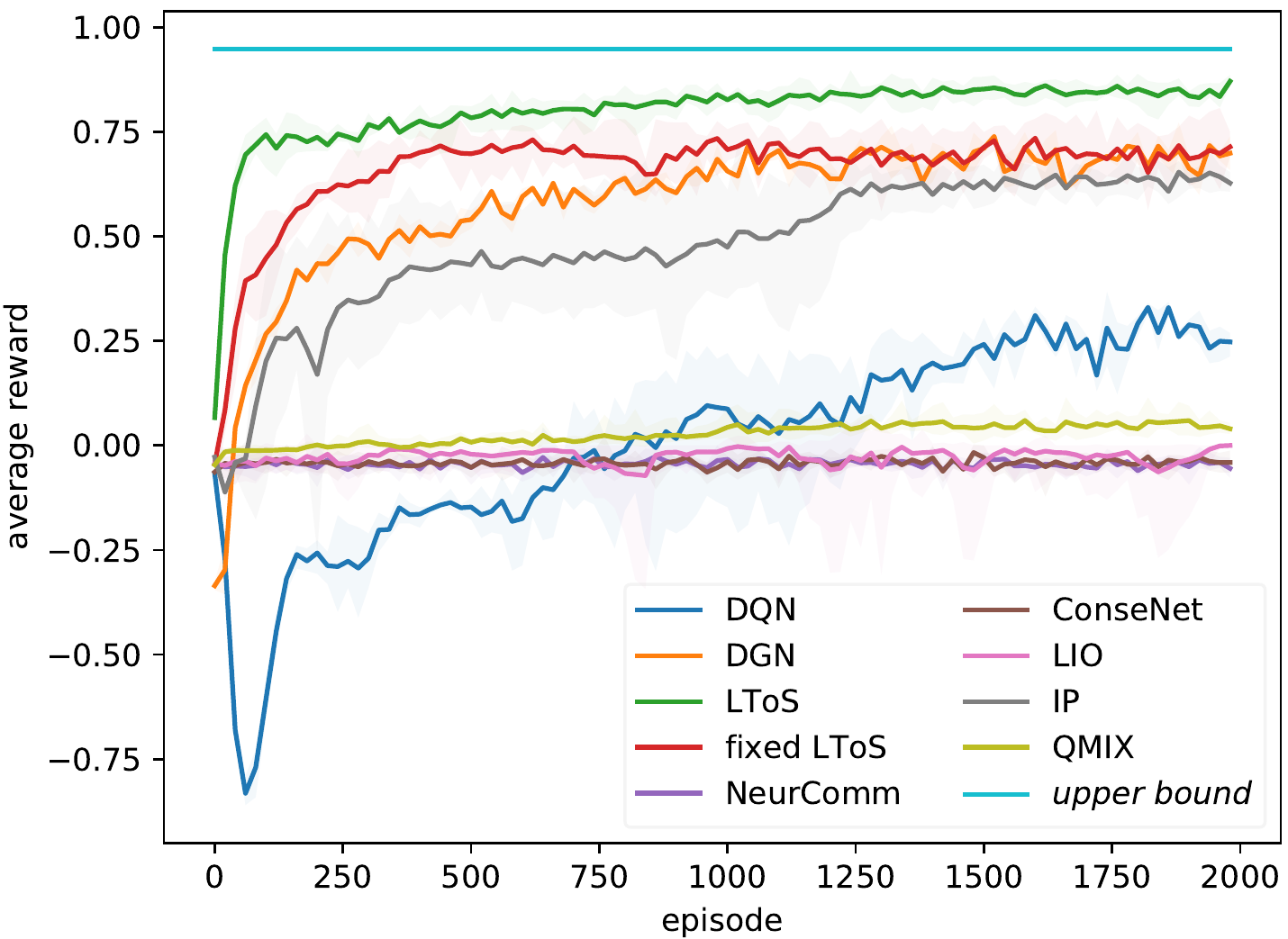}
				\label{fig:jungle-curve}}
		\end{minipage}
		\hskip 0.1cm
		\begin{minipage}[t]{0.315\textwidth}
			\setlength{\abovecaptionskip}{5pt}
			\centering
			\subfloat[][\textit{traffic}]{
				\includegraphics[width=1\textwidth]{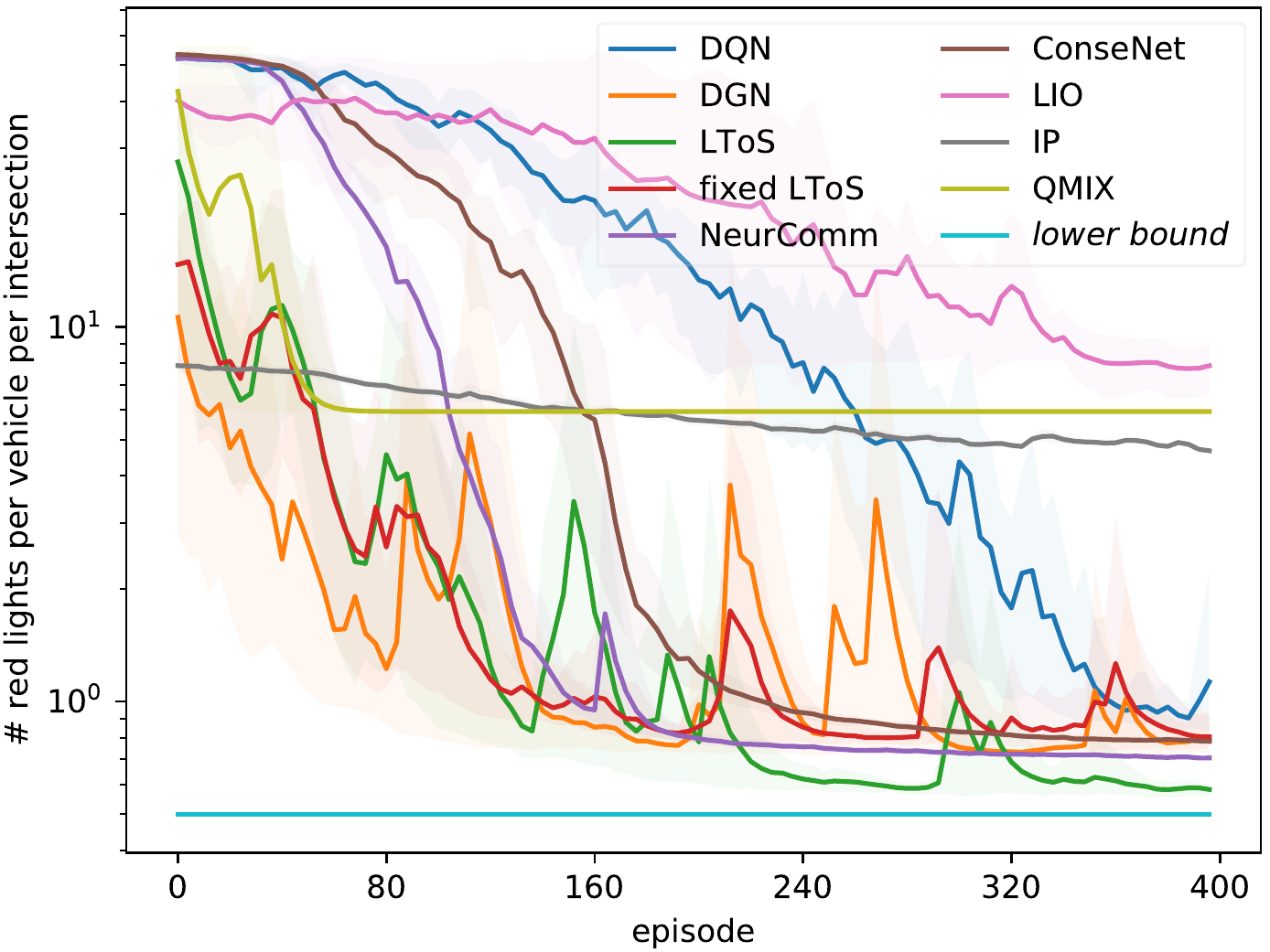}
				\label{fig:traffic-curve}}
		\end{minipage}
		\caption{Learning curves in (a) \textit{prisoner}, (b) \textit{jungle}, and (c) \textit{traffic}. All the curves are plotted using 5 training runs with different random seeds, where the solid line is the mean and the shadowed area is enclosed by the min and max value.}
		\label{fig:curves}
		\vspace{-0.2cm}
	\end{figure*}

	\subsection{Jungle}
	
	\textit{Jungle} is a scenario about moral dilemma proposed by \citet{jiang2020graph} based on MAgent \citep{zheng2017magent}. As illustrated in Figure~\ref{fig:jungle-scenario}, there are $N$ agents and $L$ stationary foods. At each timestep, each agent can attack or move to one adjacent grid. Eating (attacking food) brings a positive reward $+1$, but attacking other agents obtains a higher reward $+2$. The victim, however, suffers a negative reward $-4$, which makes each attack between agents a negative-sum action. Moreover, attacking a blank grid gets a small negative reward $-0.01$ (inhibiting excessive attacks). We follow the original setting of \citet{jiang2020graph}: map size = $30\times30$ grids, $N = 20$, $L = 12$, and the observation consists of one's coordinates and a field of $11\times11$ grids nearby. Each agent has $3$ closest agents as its neighbors. Compared to \textit{prisoner}, \textit{jungle} has much more agents, and thus JointDQN and Coco-Q are disregarded for this scenario. Another challenge of \textit{jungle} is that the network topology is dynamic since each agent can always move. Fortunately, the topology change slowly and predictably, and algorithms may get the time-varying neighbor set $\mathcal{N}_{i}$ as part of the input. Therefore, it is still likely to estimate the shaped rewards and value functions well.
	
	\begin{table}[h]
		\vspace*{-0.2cm}
		\renewcommand{\arraystretch}{1}
		\centering
		\caption{Average reward per step of all the methods in \textit{jungle}.}
		\vskip 0.1cm
		\label{jungle performance}
		\begin{small}
			\setlength{\tabcolsep}{4pt}
			\begin{tabular}{c c c c c c c c c | c}
				\toprule   
				DQN & DGN & \textit{fixed} LToS & \textbf{LToS} & NeurComm & ConseNet & LIO & IP & QMIX & \textit{upper bound} \\  
				\midrule
				0.24 & 
				0.66 & 
				$0.71$ & 
				$\boldsymbol{0.86}$ & 
				-0.05 &
				-0.04 & 
				0.00 & 
				0.63 & 
				0.04 & 
				0.95 \\ 
				\bottomrule  
			\end{tabular}
		\end{small}
	\end{table}

	Figure \ref{fig:jungle-curve} illustrates the learning curves of all the methods, and their performance after convergence is also summarized in Table~\ref{jungle performance}. NeurComm, ConseNet and QMIX do not perform well in this task. In NeurComm, each agent gets a ``delayed global information". However, a stable pattern of delayed global information cannot be formed when the communication is conducted via a dynamic topology. ConseNet is constructed on the basis of a premise of \textit{full observation}, and it can hardly learn well when the input is not only partial but also fairly varying in sequence and content. LIO agents also perform badly, since they cannot be distinguished from one another in the dynamic topology and thus fail to learn a proper incentive function. Moreover, LIO requires opponent modeling, but it is hard to simultaneously model all other agents in a dynamic environment. QMIX is free from these problems. While aiming at global optimization, like LIO, it realizes that attacking usually means a negative-sum action, but as a result, it avoids attacking as well as eating most of the time and thus only achieves a reward slightly higher than 0. Another possible reason to explain the performance QMIX is its scalability. As there are 20 agents in the scenario, it can be hard to learn the joint action-value function to directly optimize the average return \citep{qu2020intention}. Also, we can see that a fixed reward sharing scheme does not bring any gain over DGN. This is because fixed reward sharing does not adapt to the dynamic topology. By learning proper reward sharing and adjusting to changing circumstances, LToS outperforms all other baselines. Note that there is an upper bound for average reward per step for \textit{jungle}. By estimating the average distance between each agent and the food that is the closest to it at the beginning of one episode, we can give a loose upper bound around $0.95$ to reflect our improvement.
	
	\begin{figure*}[!t]
		\centering
		\setlength{\abovecaptionskip}{5pt}
		\begin{minipage}{0.189\textwidth}
			\setlength{\abovecaptionskip}{5pt}
			\centering
			\subfloat[][NeurComm]{
				\includegraphics[width=1\textwidth]{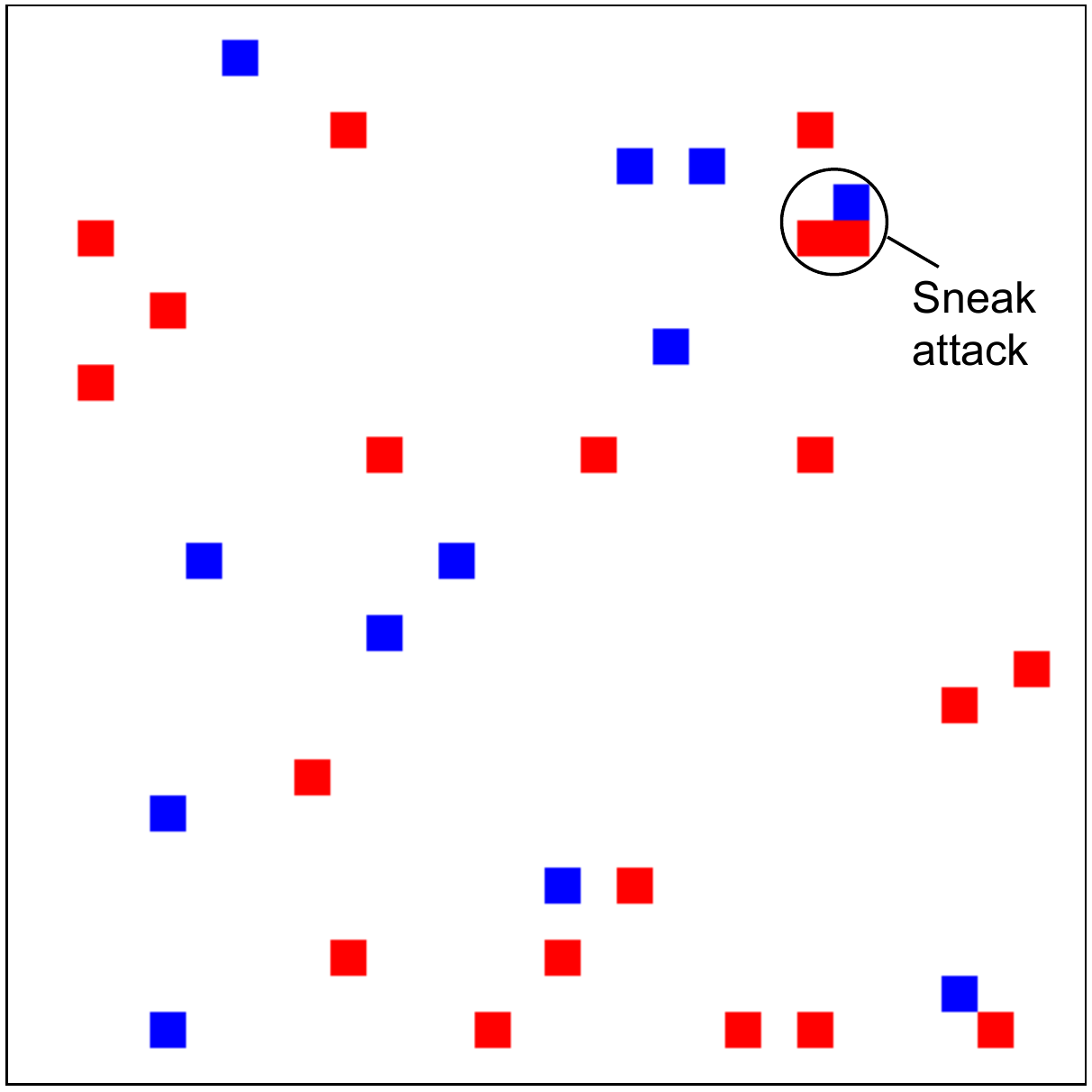}
				\label{fig:jungle-neurcomm}}
		\end{minipage}
		\begin{minipage}{0.19\textwidth}
			\setlength{\abovecaptionskip}{5pt}
			\centering
			\subfloat[][ConseNet]{
				\includegraphics[width=1\textwidth]{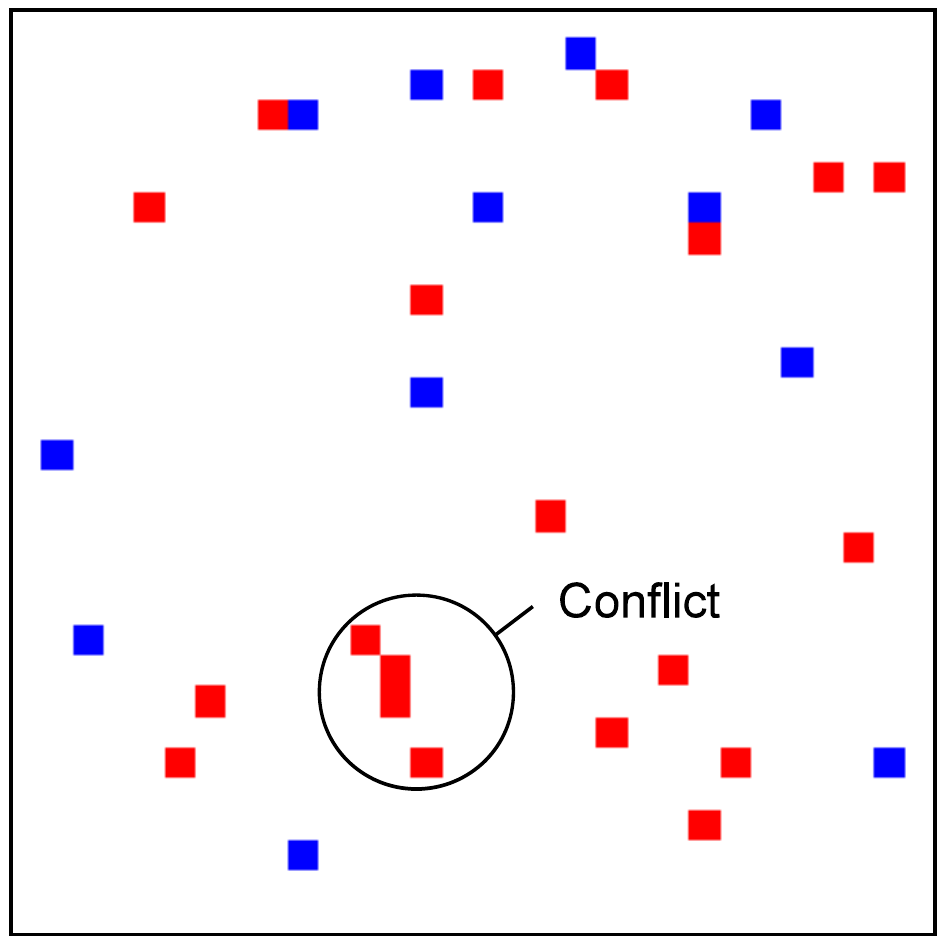}
				\label{fig:jungle-consenet}}
		\end{minipage}
		\begin{minipage}{0.19\textwidth}
			\setlength{\abovecaptionskip}{5pt}
			\centering
			\subfloat[][IP]{
				\includegraphics[width=1\textwidth]{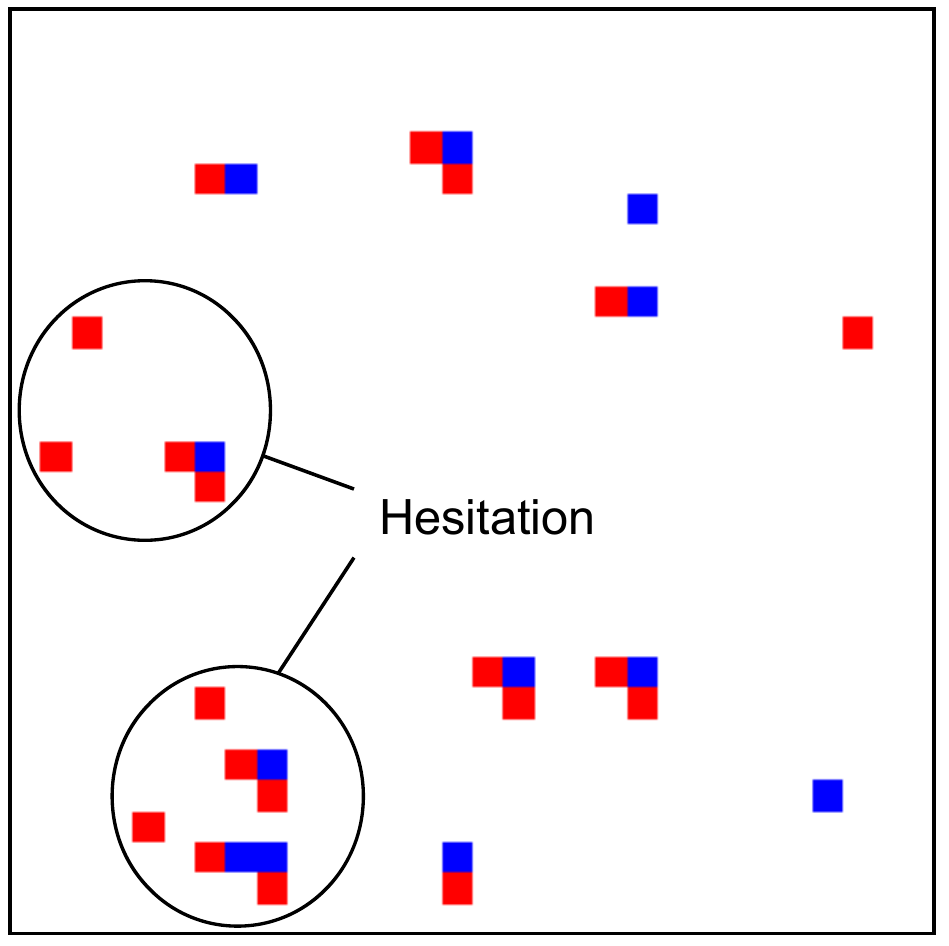}
				\label{fig:jungle-ip}}
		\end{minipage}
		\begin{minipage}{0.188\textwidth}
			\setlength{\abovecaptionskip}{5pt}
			\centering
			\subfloat[][DGN]{
				\includegraphics[width=1\textwidth]{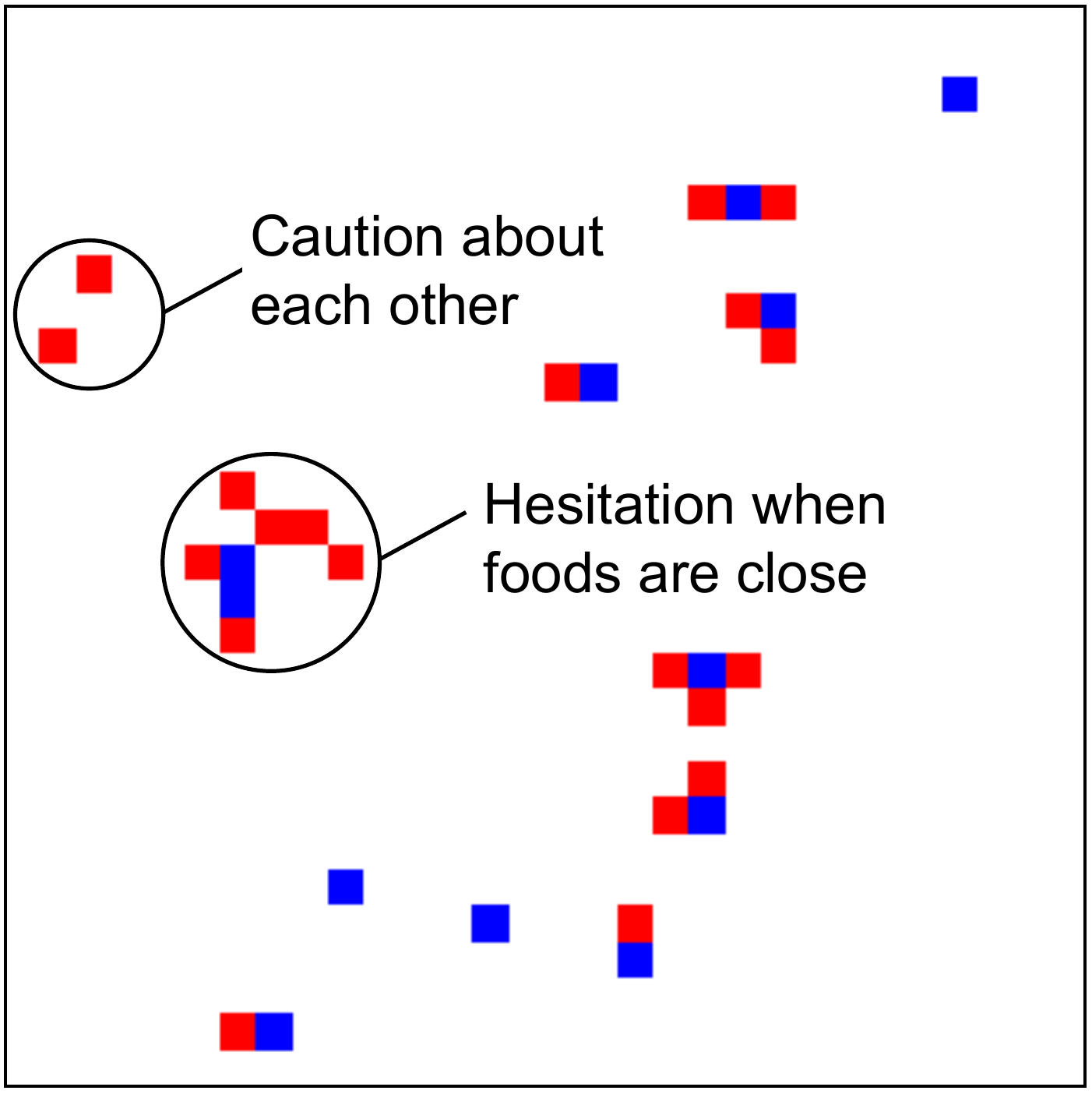}
				\label{fig:jungle-dgn}}
		\end{minipage}
		\hskip 0.3cm
		\begin{minipage}{0.189\textwidth}
			\setlength{\abovecaptionskip}{5pt}
			\centering
			\subfloat[][LToS]{
				\includegraphics[width=1\textwidth]{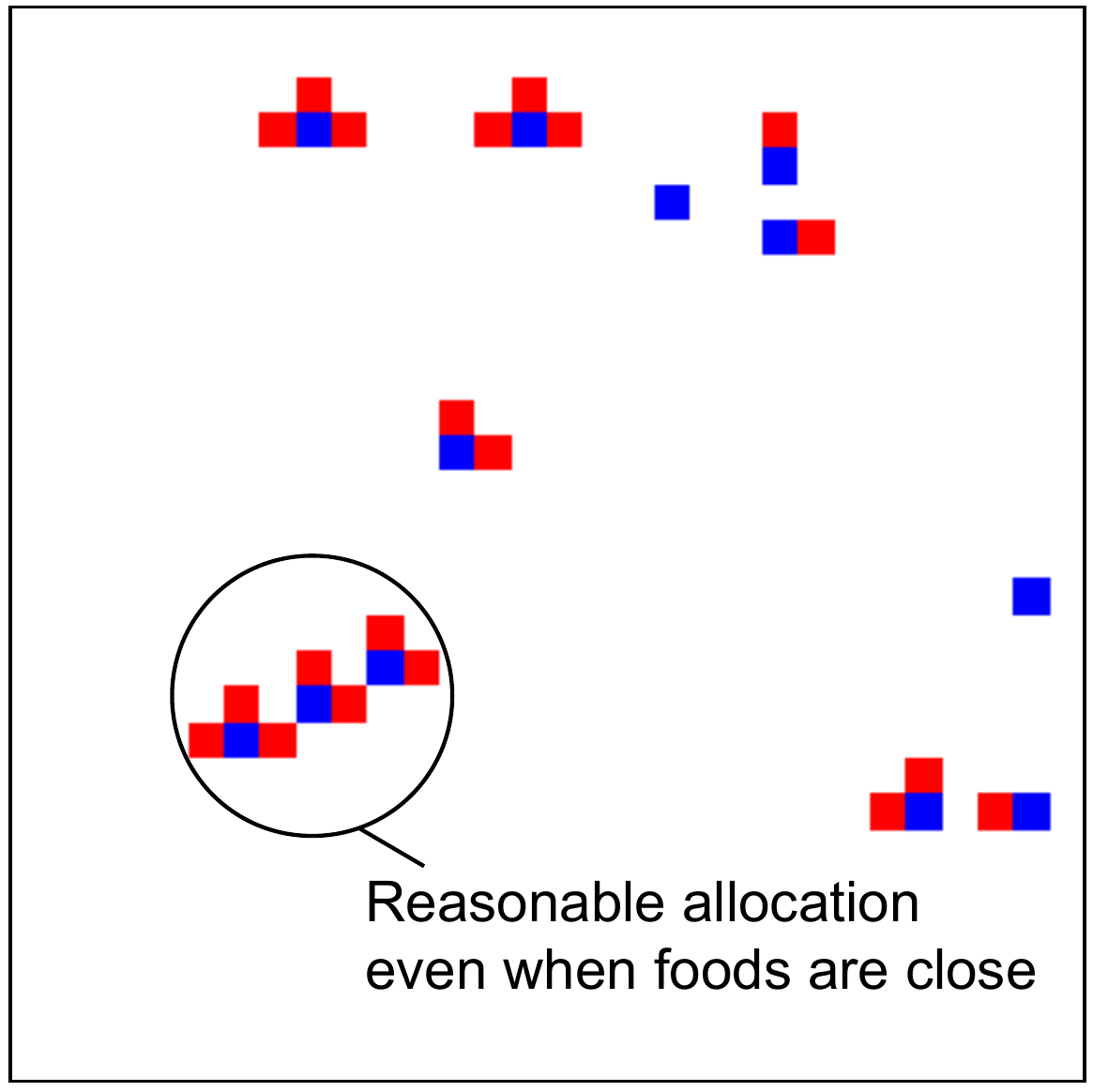}
				\label{fig:jungle-ltos}}
		\end{minipage}
		\caption{Representative behaviors of agents learned by (a) NeurComm, (b) ConseNet, (c) IP, (d) DGN, and (e) LToS in \textit{jungle}.}
		\label{fig:jungle-illustration}
		\vspace{-0.1cm}
	\end{figure*}
	
	\begin{figure}[htbp]
		\centering
		\begin{minipage}{0.4\textwidth}
			\captionsetup[subfloat]{captionskip=10pt}
			\small
			\centering
			\vspace{0.2cm}
			\subfloat[][statistics of performance]{
				\begin{small}
					\begin{tabular}{cccc|c}
						\toprule   
						\# agents & DGN & \textit{fixed} LToS & \textbf{LToS} & \textit{upper bound} \\
						\midrule
						10 & 0.52 & 0.63 & $\boldsymbol{0.71}$ & 0.93 \\
						20 & 0.66 & 0.71 & $\boldsymbol{0.86}$ & 0.95 \\
						30 & 0.77 & 0.79 & $\boldsymbol{0.88}$ & 0.96 \\
						40 & 0.84 & 0.80 & $\boldsymbol{0.90}$ & 0.96 \\
						50 & 0.86 & 0.79 & $\boldsymbol{0.91}$ & 0.97 \\
						\bottomrule
					\end{tabular}
				\end{small}
				\label{tab:jungle-scalable}}
		\end{minipage}
		\hspace{2cm}
		\begin{minipage}[h]{0.39\textwidth}
			\centering
			\subfloat[][line graph of performance]{
				\includegraphics[width=0.7\textwidth]{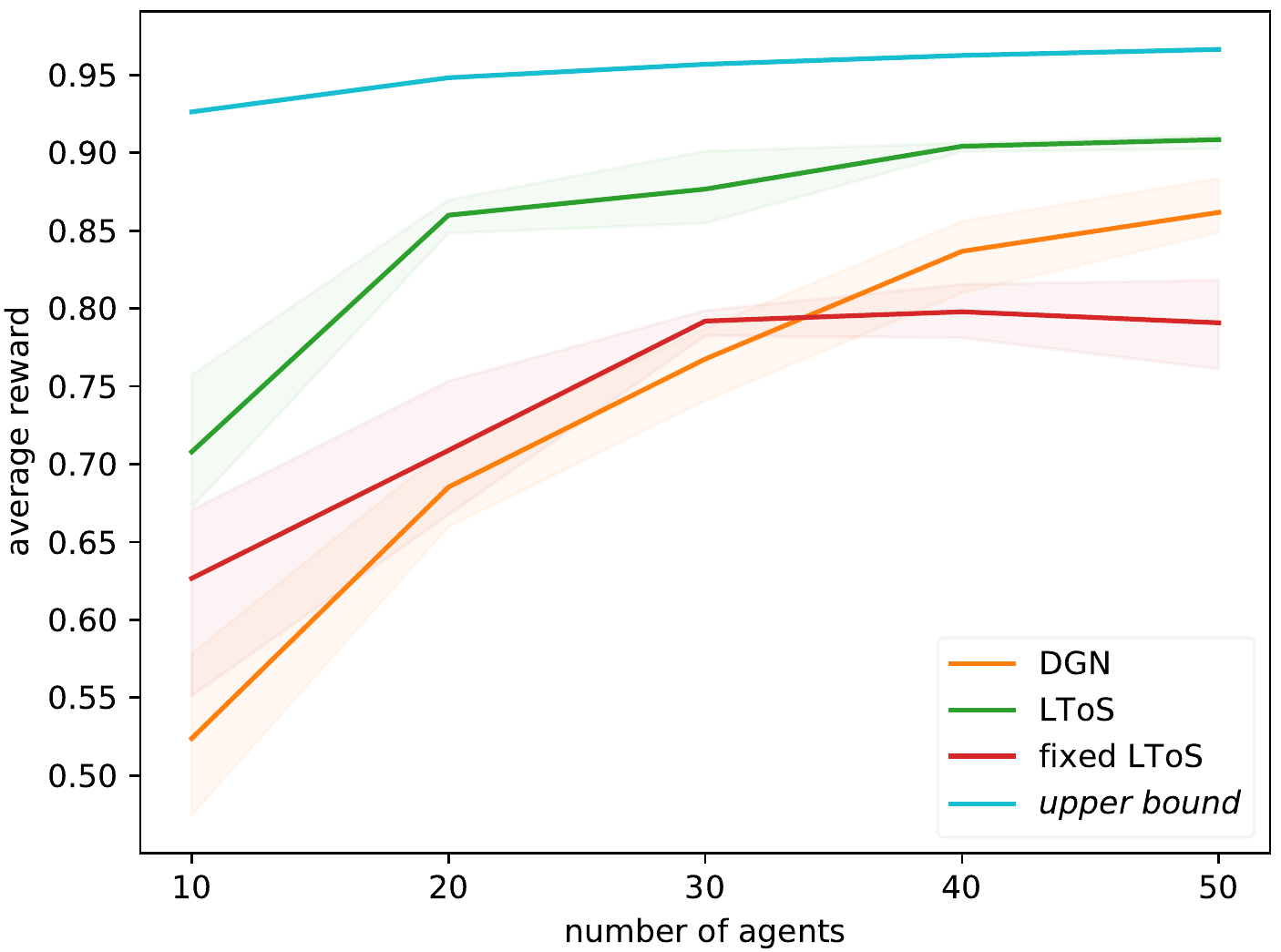}
				\label{fig:jungle-scalable}}
		\end{minipage}
		\caption{Average reward per step of methods in \textit{jungle} with different number of agents}
		\vspace{-2mm}
	\end{figure}
	
	Figure~\ref{fig:jungle-illustration} illustrates the representative behaviors of agents learned by difference methods. For NeurComm, ConseNet (and the same with LIO), most agents are not even close to the foods, so as to avoid being attacked. Even though, there are still conflict and sneak attack between agents sometimes. IP and DGN agents learn much better, but agents may still be cautious about each other, which leads to hesitation when they are near the same food. LToS agents learn to properly share the food even if the foods are close (\textit{i.e.}, agents are easy to be attacked) as depicted in Figure~\ref{fig:jungle-ltos} and demonstrate much better cooperation than the agents learned by other methods.
	The experimental results in \textit{jungle} verify that LToS can also adapts to considerably varying topology in networked MARL. 
	
	Besides, we compared LToS and ablation baselines (\textit{i.e.}, DGN and \textit{fixed} LToS) with different number of agents (\textit{i.e.}, from $10$ to $50$) to verify the scalability of LToS. All the setting remains the same except that the number of agents and food grows proportionally (\textit{i.e.}, $\nicefrac{\# agents}{\# foods} = \nicefrac{5}{3}$). As depicted in Figure~\ref{fig:jungle-scalable}, LToS can always achieve the best performance as the agent population size increases.

	\subsection{Traffic}
	
	\begin{wraptable}{r}{0.35\textwidth}
		\vspace{-.4cm}
		\setlength{\tabcolsep}{1pt}
		\setlength{\abovecaptionskip}{3pt}
		\renewcommand{\arraystretch}{0.9}
		\centering
		\caption{Statistics of traffic flows}
		\label{tab:traffic-flow}
		\begin{footnotesize}
			\begin{tabular}{@{}cc@{}}
				\toprule   
				\begin{tabular}[c]{@{}c@{}} Time\\(second)\end{tabular} & \begin{tabular}[c]{@{}c@{}} Arrival Rate\\ (vehicles/s)\end{tabular} \\
				\midrule   
				$0 - 600$ & $1$ \\
				$600 - 1,200$ & $1/4$ \\
				$1,200 - 1,800$ & $1/3$ \\
				$1,800 - 2,400$ & $2$ \\
				$2,400 - 3,000$ & $1/5$ \\
				$3,000 - 3,600$ & $1/2$ \\
				\bottomrule  
			\end{tabular}
		\end{footnotesize}
		\vspace{-0.3cm}
	\end{wraptable}
	
	
	In \textit{traffic}, as illustrated in Figure~\ref{fig:traffic-scenario} , we aim to investigate the capability of LToS in dealing with highly dynamic environment through reward sharing. We adopt the same problem setting as \citet{wei2019colight}. In a road network, each agent serves as traffic signal control at an intersection. The observation of an agent consists of a one-hot representation of its current phase (directions for red/green lights) and the number of vehicles on each incoming lane of the intersection. At each timestep, an agent chooses a phase from the pre-defined phase set for the next time interval, \textit{i.e.}, $10$ seconds. The reward is set to be the negative of the sum of the queue lengths of all approaching lanes at current timestep. The global objective is to minimize average 
	wait time of all vehicles in the road network, which is equivalent to minimizing the sum of queue lengths of all intersections over an episode \citep{zheng2019diagnosing}. The experiment was conducted on a traffic simulator, CityFlow \citep{zhang2019cityflow}. We use a $6\times6$ grid network with 36 intersections. 
	The traffic flows were generated to simulate dynamic traffic flows including both peak and off-peak period, and the statistics 
	are
	summarized in 
	Table~\ref{tab:traffic-flow}.
	
	\begin{table}[h]
		\renewcommand{\arraystretch}{1}
		\centering
		\vskip -0.2cm
		\caption{Average number of red lights one vehicle waits for at per intersection of methods in \textit{traffic}}
		\label{traffic performance}
		\setlength{\tabcolsep}{3.5pt}
		\begin{small}
			\begin{tabular}{c | c c c c c c c c c | c}
				\toprule   
				Network & DQN & DGN & \textit{fixed} LToS & \textbf{LToS} & NeurComm & ConseNet & LIO & IP & QMIX & \textit{lower bound} \\  
				\midrule
				$6\times6$ & 0.90 & 0.78 & 0.80 & \textbf{0.58} & 0.71 & 0.78 & 7.74 & 4.67 & 5.94 & 0.50 \\
				Shenzhen & 13.99 & 2.02 & 1.94 & \textbf{1.71} & 2.27 & 7.63 & 19.15 & 5.52 & 19.74 & 1.50 \\

				\bottomrule  
			\end{tabular}
		\end{small}
	\end{table}
	
	For better demonstration, we choose to show the normalized metric of wait time: the average number of red lights one vehicle waits for at per intersection, 
	and we can also give a loose lower bound $0.50$ to reflect our improvement. Figure~\ref{fig:traffic-curve} shows the learning curves of all the methods in terms of that in logarithmic form. 
	The performance after convergence is summarized in Table~\ref{traffic performance}, where LToS outperforms all other methods. LToS outperforms DGN, which demonstrates the reward sharing scheme learned by the high-level policy indeed helps to improve the cooperation of agents. Without the high-level policy, \textit{i.e.}, given fixed sharing weights, \textit{fixed} LToS does not perform well in dynamic environment. This indicates the necessity of the high-level policy. The performance of IP agents increases so slowly that they cannot converge efficiently. Although NeurComm and ConseNet both take advantage of RNN for partially observable environments, LToS still outperforms these methods, which verifies the great improvement of LToS in networked MARL. QMIX is confined to suboptimality \citep{mahajan2019maven}. As observed in the experiment, QMIX tries to release traffic flows from one direction while stopping flows from the other direction all the time, because this will only make two rows of intersections on the border blocked but keep most of the intersections from any traffic jam all the time. However, the global optimality actually does not need to be constructed on the sacrifice of anyone.
	Some similar thing happens to LIO. It is likely because LIO contains some sensitive parameters \citep{yang2020learning} and agents are hard to learn and coordinate their incentive functions since the original reward functions change acutely once an improper operation causes traffic congestion. An introduction of some explicit coordination mechanism may also alleviate the problem, like that of NeurComm and ConseNet. 
	
	\begin{figure*}[!t]
		\setlength{\abovecaptionskip}{5pt}
		\centering
		\begin{minipage}{0.27\textwidth}
			\setlength{\abovecaptionskip}{5pt}
			\centering
			\subfloat[][temporal pattern]{
				\includegraphics[width=1\textwidth]{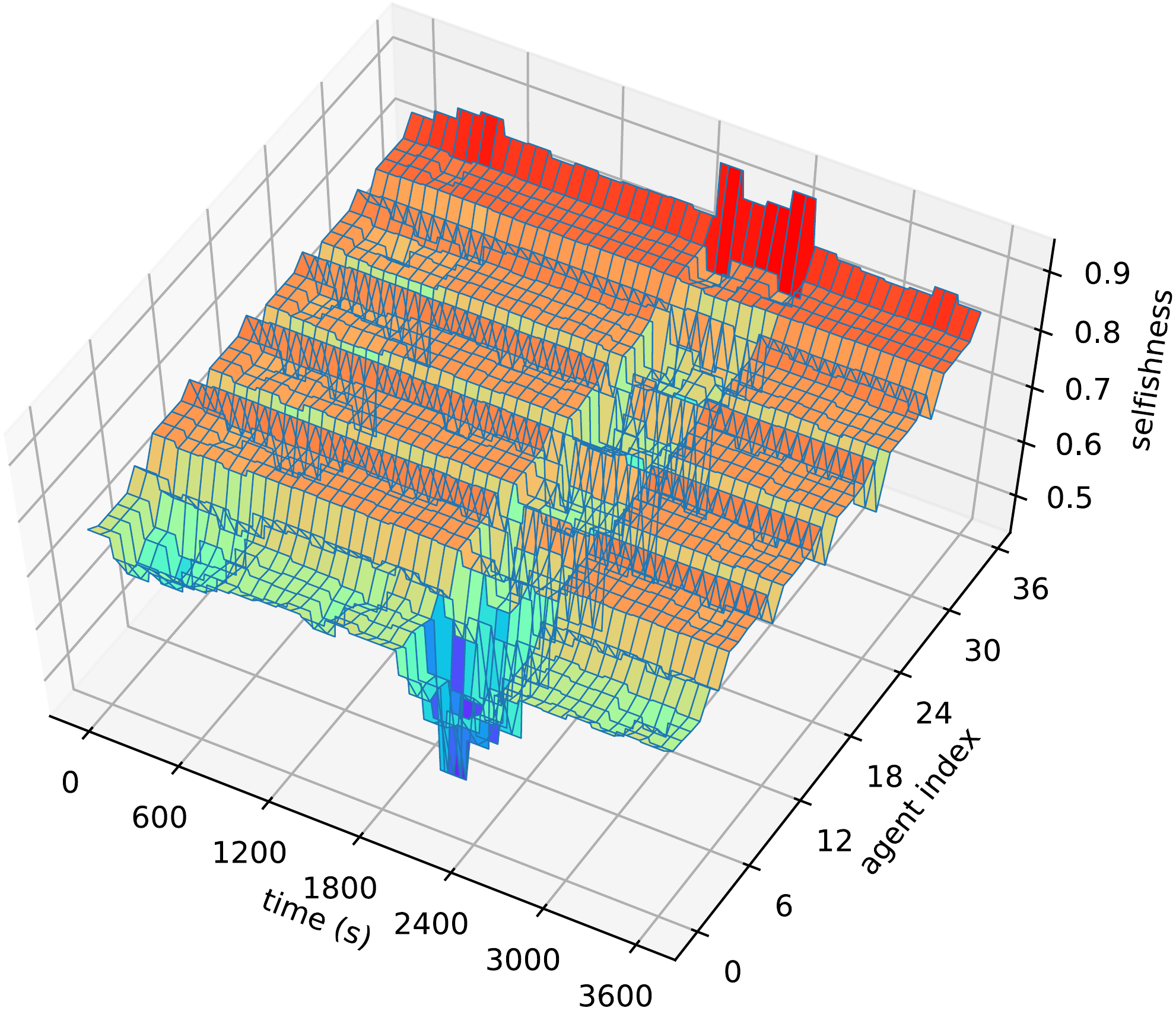}
				\label{fig:temporal-selfish}}
		\end{minipage}
		\hspace{0.4cm}
		\begin{minipage}{0.48\textwidth}
			\setlength{\abovecaptionskip}{5pt}
			\centering
			\subfloat[][spatial pattern ]{
				\includegraphics[width=0.9\textwidth]{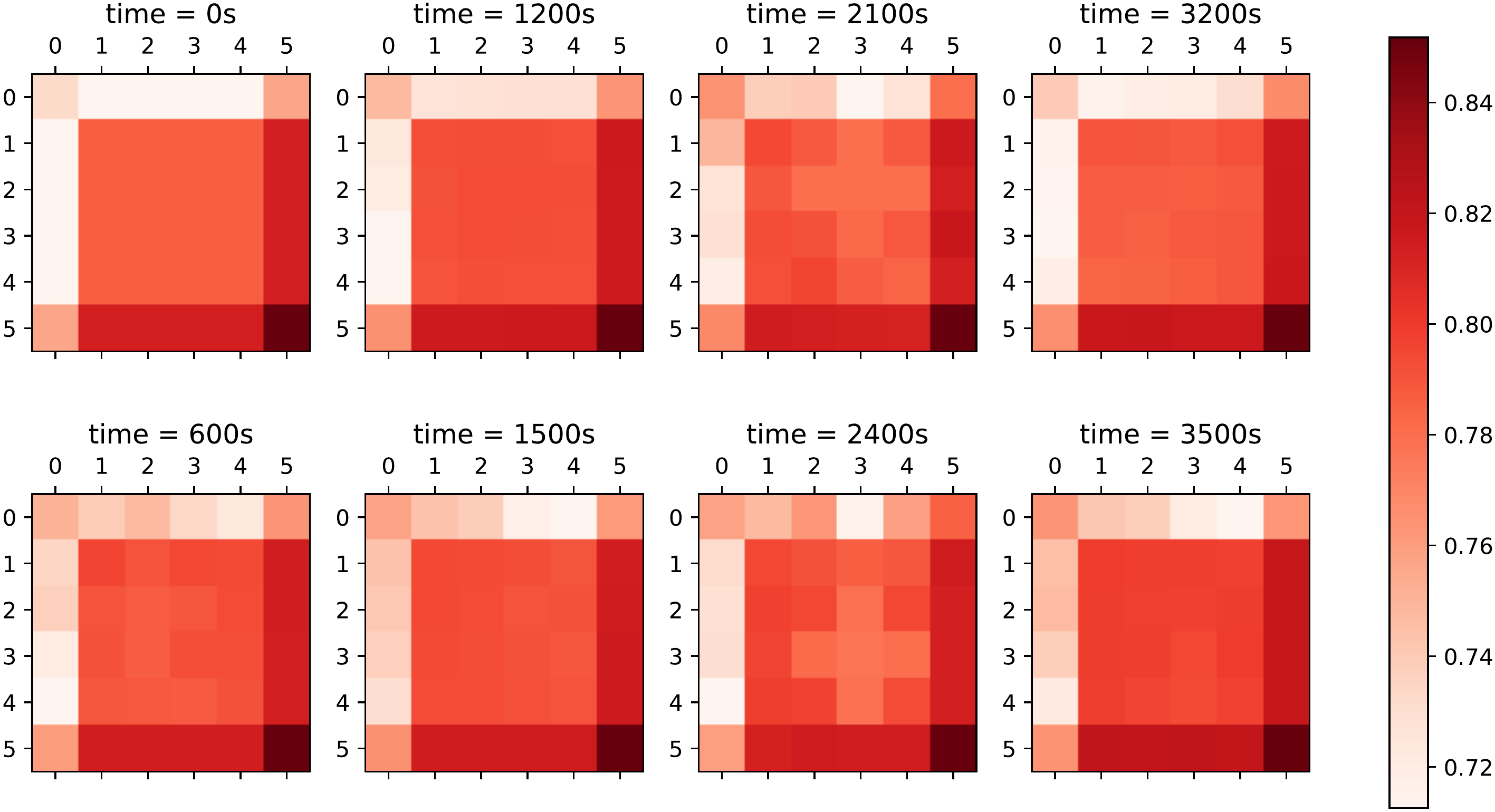}
				\label{fig:spatial-selfish}}
		\end{minipage}
		\caption{Patterns of \textit{selfishness} in \textit{traffic}.}
		\vspace{-0.4cm}
	\end{figure*}
	
	We visualize the variation of \textit{selfishness} of all agents during an episode in \textit{traffic} in Figure~\ref{fig:temporal-selfish} and \ref{fig:spatial-selfish}. Figure~\ref{fig:temporal-selfish} depicts the temporal variance of selfishness for each agent. For most agents, there are two valleys occurred exactly during two peak periods (\textit{i.e.}, $0-600$s and $1,800-2,400$s). This is because for heavy traffic agents need to cooperate more closely, which can be induced by being less selfish. We can see this from the fact that selfishness is even lower in the second valley where the traffic is even heavier (\textit{i.e.}, $2$ \textit{vs.} $1$ vehicles/s). Therefore, this demonstrates that the agents learn to adjust their extent of cooperation to deal with dynamic environment by controlling the sharing weights. Figure~\ref{fig:spatial-selfish} shows the spatial pattern of selfishness at different timesteps, where the distribution of agents is the same as the road network in Figure~\ref{fig:traffic-scenario}. The edge and inner agents tend to have very different selfishness. In addition, inner agents keep their selfishness more uniform during off-peak periods, while they diverge and present cross-like patterns during peak periods. 
	
	\begin{wrapfigure}{r}{0.3\textwidth}
		\setlength{\abovecaptionskip}{5pt}
		\vspace{-0.35cm}
		\centering
		\includegraphics[width=0.28\textwidth]{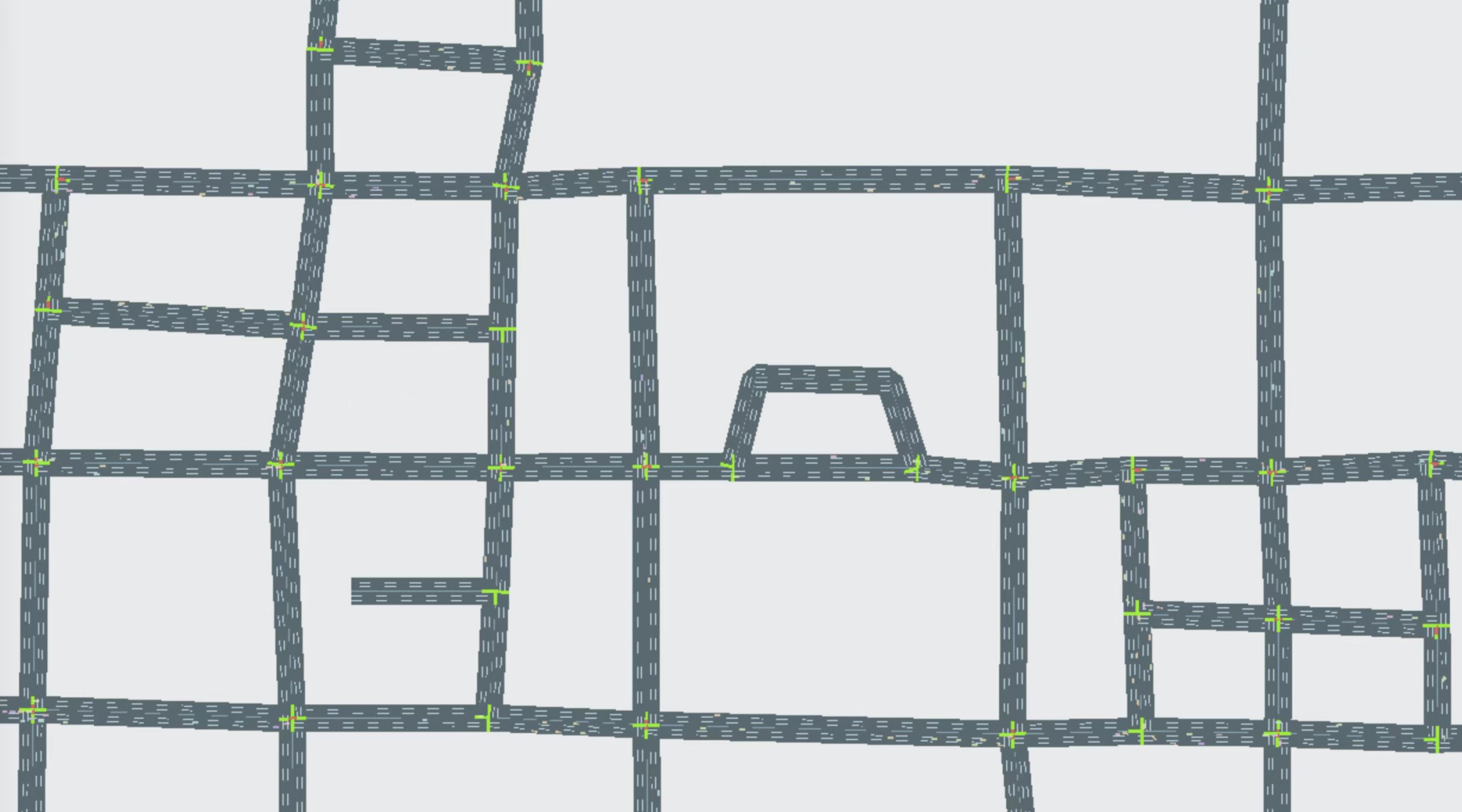}
		\caption{Shenzhen network}
		\label{fig:traffic-shenzhen}
		\vspace{-0.5cm}
	\end{wrapfigure}
	
	This shows that handling heavier traffic requires more diverse reward sharing schemes among agents to promote more sophisticated cooperation. 
	In addition, we use another network that is part of Shenzhen, China with 33 intersections, and real one-hour traffic flows \citep{RoadnetSZ}. 
	The performance is also summarized in Table~\ref{traffic performance}.
	The experimental results in \textit{traffic} verify that LToS can also handle highly dynamic environment in networked MARL.

	\section{Conclusion}
	In this paper, we proposed LToS, a hierarchically decentralized framework for networked MARL. LToS enables agents to share reward with neighbors so as to encourage agents to cooperate on the global objective
	through collectives.
	For each agent, the high-level policy learns how to share reward with neighbors to decompose the global objective, while the low-level policy learns to optimize the local objective induced by the high-level policies in the neighborhood. Experimentally, we demonstrate that LToS outperforms existing methods in both social dilemma and networked MARL scenario
	across scales.

	\bibliography{ref}
	\bibliographystyle{plainnat}

	\clearpage
	\appendix

	\section{Proofs}
	\label{app:proofs}
	
	\subsection{Proof of Proposition \ref{prop:1}}
	\textbf{Proposition \ref{prop:1}.} 
	\textit{Given $\boldsymbol{\pi}$, $V_{\mathcal{V}}^{\boldsymbol{\phi}}(s;\boldsymbol{\pi})$ and $Q_{\mathcal{V}}^{\boldsymbol{\phi}}(s, \boldsymbol{w};\boldsymbol{\pi})$ are respectively the value function and action-value function of $\boldsymbol{\phi}$.}
	
	\label{proof:1}
	\begin{proof}
		Let $r_{i}^{\boldsymbol{\phi}} \doteq \sum_{\boldsymbol{a}} \boldsymbol{\pi}(\boldsymbol{a}|s,\boldsymbol{w}) r_{i}^{\boldsymbol{w}}$ and $p_{{w}}(s^\prime|s,\boldsymbol{w}) \doteq \sum_{\boldsymbol{a}} \boldsymbol{\pi}(\boldsymbol{a}|s,\boldsymbol{w}) p_a(s^\prime|s,\boldsymbol{a})$
		in contrast with $p_{{a}}$
		. As commonly assumed the reward is deterministic given $s$ and $\boldsymbol{a}$, from (\ref{v_i_phi}), we have,
		\begin{align}
			v_{i}^{\boldsymbol{\pi}}(s;\boldsymbol{\phi}) &= \sum_{\boldsymbol{w}} \boldsymbol{\phi}(\boldsymbol{w}|s) \sum_{\boldsymbol{a}}  \boldsymbol{\pi}(\boldsymbol{a}|s,\boldsymbol{w}) [r_{i}^{\boldsymbol{w}} + \sum_{s^\prime} p_a(s^\prime|s,\boldsymbol{a}) \gamma v_{i}^{\boldsymbol{\pi}}(s^\prime;\boldsymbol{\phi})] \notag \\ \label{v_i}
			&= \sum_{\boldsymbol{w}} \boldsymbol{\phi}(\boldsymbol{w}|s) \sum_{s^\prime}p_{{w}}(s^\prime|s,\boldsymbol{w}) [r_{i}^{\boldsymbol{\phi}}+\gamma v_{i}^{\boldsymbol{\pi}}(s^\prime;\boldsymbol{\phi})],
		\end{align}
		where $p_{{w}} \in \mathcal{P}_{{w}}: \mathcal{S \times W \times S} \rightarrow [0,1]$ describes the state transitions given $\boldsymbol{\pi}$. 
		
		Let $r_{\mathcal{V}}^{\boldsymbol{\phi}} \doteq \sum_{i \in \mathcal{V}} r_{i}^{\boldsymbol{\phi}}$, and from (\ref{v_i}) we have 
		\begin{align*}
			V_{\mathcal{V}}^{\boldsymbol{\phi}}(s;\boldsymbol{\pi}) &= \sum_{i \in \mathcal{V}} \sum_{\boldsymbol{w}} \boldsymbol{\phi}(\boldsymbol{w}|s) \sum_{s^\prime} p_{{w}}(s^\prime|s,\boldsymbol{w}) [r_{i}^{\boldsymbol{\phi}} + \gamma v_{i}^{\boldsymbol{\pi}}(s^\prime;\boldsymbol{\phi})] \notag \\
			&= \sum_{\boldsymbol{w}} \boldsymbol{\phi}(\boldsymbol{w}|s) \sum_{s^\prime} p_{{w}}(s^\prime|s,\boldsymbol{w})[\sum_{i \in \mathcal{V}} r_{i}^{\boldsymbol{\phi}} + \gamma \sum_{i \in \mathcal{V}} v_{i}^{\boldsymbol{\pi}}(s^\prime;\boldsymbol{\phi})] \notag \\
			&= \sum_{\boldsymbol{w}} \boldsymbol{\phi}(\boldsymbol{w}|s) \sum_{s^\prime} p_{{w}}(s^\prime|s,\boldsymbol{w}) [ r_{\mathcal{V}}^{\boldsymbol{\phi}} + \gamma  V_{\mathcal{V}}^{\boldsymbol{\phi}}(s^\prime; \boldsymbol{\pi})],
		\end{align*}
		and similarly,
		\begin{align*}
			Q_{\mathcal{V}}^{\boldsymbol{\phi}}(s, \boldsymbol{w};\boldsymbol{\pi}) &= \sum_{i \in \mathcal{V}} \sum_{s^\prime} p_{{w}}(s^\prime|s,\boldsymbol{w})[r_{i}^{\boldsymbol{\phi}} + \gamma \sum_{\boldsymbol{w}^\prime} \phi(\boldsymbol{w}^\prime|s^\prime) v_{i}^{\boldsymbol{\pi}}(s^\prime;\boldsymbol{w}^\prime,\boldsymbol{\phi})] \notag \\
			&  = \sum_{s^\prime} p_{{w}}(s^\prime|s,\boldsymbol{w}) [r_{\mathcal{V}}^{\boldsymbol{\phi}} + \gamma \sum_{\boldsymbol{w}^\prime} \boldsymbol{\phi}(\boldsymbol{w}^\prime|s^\prime) Q_{\mathcal{V}}^{\boldsymbol{\phi}}(s^\prime, \boldsymbol{w}^\prime; \boldsymbol{\pi})].
		\end{align*}
		Moreover, from the definitions of $r_{i}^{\boldsymbol{w}}$ and $r_{i}^{\boldsymbol{\phi}}$ we have
		\begin{align*}
			r_{\mathcal{V}}^{\boldsymbol{\phi}} & 
			= \sum_{\boldsymbol{a}} \boldsymbol{\pi}(\boldsymbol{a}|s,\boldsymbol{w}) \sum_{i \in \mathcal{V}}  r_{i}^{\boldsymbol{w}} \\
			& = \sum_{\boldsymbol{a}} \boldsymbol{\pi}(\boldsymbol{a}|s,\boldsymbol{w}) \sum_{i \in \mathcal{V}}\sum_{j \in \mathcal{N}_{i}} w_{ji} r_{j} \\
			&= \sum_{\boldsymbol{a}} \boldsymbol{\pi}(\boldsymbol{a}|s,\boldsymbol{w}) \sum_{(i,j) \in \mathcal{D}} w_{ij} r_i 
			= \sum_{\boldsymbol{a}} \boldsymbol{\pi}(\boldsymbol{a}|s,\boldsymbol{w}) \sum_{i \in \mathcal{V}} r_i.
		\end{align*}
		Thus, given $\boldsymbol{\pi}$, $V_{\mathcal{V}}^{\boldsymbol{\phi}}(s)$ and $Q_{\mathcal{V}}^{\boldsymbol{\phi}}(s,\boldsymbol{w})$ are respectively the value function and action-value function of $\boldsymbol{\phi}$ in terms of the sum of expected cumulative rewards of all agents, \textit{i.e.}, the global objective. 
	\end{proof}
	
	\subsection{Proof of Proposition \ref{prop:3}}
	\textbf{Proposition \ref{prop:3}.} 
	\textit{The joint high level policy $\boldsymbol{\phi}$ can be learned in a decentralized manner, and the decentralized high-level policies of all agents form a mean-field approximation of $\boldsymbol{\phi}$.}
	
	\label{proof:3}
	
	First, we introduce one definition and one lemma.
	
	\begin{definition}[\textbf{Markov Random Field}]
		\label{definition:1}
		A Markov Random Field (MRF) is a graph $\mathcal{G = (V, E)}$ that satisfies:
		\begin{equation}
			P(X_i|\{X_j\}_{j \in \mathcal{V} \backslash \{i\}} ) =
			P(X_i|\{X_j\}_{j \in \mathcal{N}_i} )
		\end{equation}
	\end{definition}
	where $X_i$ is some random variable associated with node $i, \forall i \in \mathcal{V}$.
	
	\begin{lemma}[\textbf{Hammersley–Clifford Theorem}]
		\label{lemma:1}
		A probability distribution that has a strictly positive mass or density satisfies one of the Markov properties with respect to an undirected graph $\mathcal{G}$ if and only if it is a Gibbs random field, i.e., its density can be factorized over the cliques (or complete subgraphs) of the graph. \citep{Hammersley1971MarkovFO}
	\end{lemma}
	
	
	Now we begin the proof of Proposition \ref{prop:3}.
	
	\begin{proof}
		Let $d_{ij} \in \mathcal{D}$ serve as a vertex with action $w_{ij}$ and reward $w_{ij} r_i$ in a new graph $\mathcal{G}^\prime$. Each vertex has its own local policy $\phi_{ij}(w_{ij}|s)$.
		Note that in the sense of mean-field approximation, we focus on neighbors and find a MRF: each $w_{ij}$ needs and only needs to
		be determined considering other $\{w_{ik} | {k \in \mathcal{N}_i \backslash \{j\}} \}$, because their actions are subject to the constraint $\sum_{j \in \mathcal{N}_{i}} w_{ij} = 1$. It accords with the adjacency relationship in $\mathcal{G^\prime}$.
		According to Lemma \ref{lemma:1}, it is also a Gibbs random field.
		
		Now we consider the cliques that we factor $\boldsymbol{\phi}(\boldsymbol{w}|s)$ over. For $\forall i \in \mathcal{V}$, $\{d_{ij}|j \in \mathcal{N}_{i}\}$ should form a complete subgraph in $\mathcal{G}^\prime$.
		Note that $d_{ij} \in \mathcal{G}^\prime$ connects to $\{d_{ik}|k \in \mathcal{N}_{i}\backslash\{j\}\}$ and $\{d_{kj}|k \in \mathcal{N}_{j}\backslash\{i\}\}$, but only the former will form the maximal clique.
		Therefore, we have $\boldsymbol{\phi}(\boldsymbol{w}|s) \approx \prod_{i \in \mathcal{V}} \phi_i(w_{i}^{\text{out}}|s)$.
		Note that technically each agent $i$ can determine $\{w_{ij} | j \in \mathcal{N}_{i}\}$ simultaneously. We allow agent $i$ to take charge of $\phi_i(w_{i}^{\text{out}}|s)$ as its high-level policy which is a joint policy of the complete subgraph in $\mathcal{G^\prime}$, so that we can turn the view back to $\mathcal{G}$ from $\mathcal{G^\prime}$ and verify each agent's independence in the high level. 
		
		Besides, from Proposition \ref{prop:1}, we approximately have:
		$
		q_{i}^{\phi_i}(s,w_{i}^{\text{out}};\pi_{\mathcal{N}_{i}}) = v_{i}^{\pi_i}(s;w_{i}^{\text{in}},\phi_{\mathcal{N}_{i}})
		$,
		where $q_{i}^{\phi_i}$ is the action-value function of $\phi_i$ given $\pi_{\mathcal{N}_{i}}$, $v_{i}^{\pi_i}$ is the value function of $\pi_i$ given $\phi_{\mathcal{N}_{i}}$ and conditioned on $w_i^\text{in}$.
		Let $Q_{\mathcal{N}_i}^{\boldsymbol{\phi}}(s, \boldsymbol{w};\boldsymbol{\pi}) \doteq \sum_{j \in \mathcal{N}_i} v_{j}^{\boldsymbol{\pi}}(s; w_j^{\text{in}}, \boldsymbol{\phi})$. Note that $\phi_i$ optimizes $Q_{\mathcal{N}_i}^{\boldsymbol{\phi}}(s, \boldsymbol{w};\boldsymbol{\pi})$, 
		because only elements in $\{v_{j}^{\boldsymbol{\pi}}(s; w_j^{\text{in}}, \boldsymbol{\phi}) | {j \in \mathcal{N}_i }\}$ correlate with $w_i^{\text{out}}$, 
		while those in $\{v_{j}^{\boldsymbol{\pi}}(s; w_j^{\text{in}}, \boldsymbol{\phi}) | {j \in \mathcal{V} \backslash \mathcal{N}_i }\}$ do not.
		After taking this uncorrelated set into account, we have an equivalent optimization of $Q_{\mathcal{V}}^{\boldsymbol{\phi}}(s, \boldsymbol{w};\boldsymbol{\pi})$, \textit{i.e.}, the global objective.
		Therefore, each decentralized high-level policy shares the same optimization objective as the global one, and we can factorize $J_{\boldsymbol{\phi}}$ into $\{ J_{\phi_i} | {i \in \mathcal{V}}\}$.
	\end{proof}
	
	This proposition gives a factorization which is different from existing studies. First, our factorization differs from $\pi(\boldsymbol{a}|s) = \Pi_{i=1}^N \pi(a_i|s)$ \citep{zhang2018fully}, since each agent needs to make decisions considering other agents’ plan. Also in contrast to \cite{qu2020intention}, they parameterize intention propagation by GNN and other neural networks to factorize the joint policy thoroughly, while we accept incomplete factorization and group indecomposable cliques by each agent to form high-level policies that are also decentralized and independent of each other. 

	\section{Algorithm}
	\label{algorithm}
	
	We describe LToS as Algorithm \ref{alg:ltos}. The code is available at \url{https://github.com/PKU-AI-Edge/RoadnetSZ}.
	\begin{algorithm}[htbp]
		\caption{LToS}
		\label{alg:ltos}
		\begin{algorithmic}[1]
			\State Initialize $\phi_i$ parameterized by $\theta_i$ and $\pi_i$ parameterized by $\mu_i$ for each agent $i$ ($\phi_i$ is learned by DDPG and $\pi_i$ is learned by DGN, and they share the Q-network)
			\For{$t = 0$ to $T$}
			\For{each agent $i$}
			\State exchange observations and get $o_i$
			\State $w_{i}^{\text{out}} \gets \phi_i(o_{i})$ with exploration 
			\State exchange $w_{i}^{\text{out}}$ and get $w_{i}^{\text{in}}$ 
			\State $a_{i} \gets \pi_i(o_{i};w_{i}^{\text{in}})$ with exploration 
			\State execute $a_{i}$, obtain $r_{i}$, and transition to $o_{i}^\prime = o_{i,t+1}$
			\State exchange $r_{i}$ and get $r_{i}^{{w}}$ 
			\State store $(o_{i}, w_{i}^{\text{in}}, a_{i}, r_{i}^{{w}}, o_{i}^\prime, \mathcal{N}_{i})$ in $\mathcal{B}_i$ 
			\EndFor
			\If{$t\bmod update\_frequency = 0$}
			\For{each agent $i$}
			\State sample a minibatch from replay buffer $\mathcal{B}_i$: $D=\{(o_i, w_i^{\text{in}}, a_i, r_i^{\boldsymbol{w}}, o_i^\prime, \mathcal{N}_i)\}$ 
			\State exchange $w_i^{\text{out}'} \gets \phi_i^\prime(o_i^\prime)$ and get $w_i^{\text{in}'}$ 
			\State set $y_i \gets r_i^{\boldsymbol{w}} + \gamma q_i^{\pi_i'}(o_i^\prime,a_i^\prime;w_i^{\text{in}^\prime})|_{a_i^\prime = \pi_i'(o_i^\prime;w_i^{\text{in}'})}$ 
			\State update $\mu_i$ by $\nabla_{\mu_i} \mathbb{E}_D (y_i-q_i^{\pi_i}(o_i,a_i;w_i^{\text{in}}))^2$
			\State exchange $w_i^{\text{out}} \gets \phi_i(o_i)$ and get $w_i^{\text{in}}$ 
			\State compute $g_i^{\text{in}}=\nabla_{w_i^{\text{in}}}q_i^{\pi_i}(o_i,\arg\max_{a_i} q_i^{\pi_i};w_i^{\text{in}})$
			\State exchange $g_i^{\text{in}}$ and get gradient $g_i^{\text{out}}$ for $w_i^{\text{out}}$
			\State update $\theta_i$ by $\frac{1}{|D|} \sum_{o_i \in D} (\nabla_{\theta_i} \phi_i(o_i))^{\mathsf{T}} g_i^{\text{out}}$
			\State softly update target networkrs $\theta_i'$ and $\mu_i'$
			\EndFor
			\EndIf
			\EndFor
		\end{algorithmic}
	\end{algorithm}

	\section{Discussions on Training LToS}
	\label{discussion}
	As a hierarchically decentralized MARL framework, LToS brings some challenges for training. 
	
	\textsf{Selfishness Initializer.} On the basis of a straightforward idea that one should generally focus more on its own reward than that of others when optimizing its own policy, the initial output of each high-level policy network is supposed to be higher on the sharing weight of its own than others. We choose to predetermine the initial selfishness to learn the high-level policy effectively. However, with normal initializers, the output of the high-level policy network will be evenly distributed initially. Therefore, we use a special \textit{selfishness initializer} for each high-level policy network instead. As we use the softmax to produce the weights, which guarantees the constraint: $\sum_{j\in \mathcal{N}_i} w_{ij} = 1, \forall i \in \mathcal{V}$, we specially set the bias of the last fully-connected layer so that each decentralized high-level policy network tends to keep for itself the same reward proportion as the given selfishness initially. The rest of reward is still evenly distributed among neighbors. LToS learns started from such initial weights, while \textit{fixed} LToS uses such weights throughout each experiment. Moreover, we use grid search to find the best selfishness for \textit{fixed} LToS in \textit{traffic} and \textit{jungle}. For \textit{prisoner} we deliberately set the selfishness to $0.5$ so that \textit{fixed} LToS directly optimizes the average return.
	
	\textsf{Unified Pseudo Random Number Generator.} LToS is learned in a decentralized manner. This incurs some difficulty for experience replay. As each agent $i$ needs $w_i^{\text{in}}$ to update network weights for both high-level and low-level policies, it should sample from its buffer a batch of experiences where each sampled experience should be synchronized across the batches of all agents (\textit{i.e.}, the experiences should be collected at a same timestep). To handle this, all agents can simply use a unified pseudo random number generator and the same random seed. 
	
	\textsf{Different Time Scales.} As many hierarchical RL methods do, we can set the high-level policy to running at a slower time scale than the low-level one. Proposition 1 
	still holds if we expand $v_i^{\boldsymbol{\pi}}$ for more than one step forward. Assuming the high-level policy runs every $M$ timesteps, we can fix $w_i^{\text{out}} = w^{\text{out},t+1} = \cdots = w^{\text{out},t+M-1}$. $M$ is referred to as \textit{action interval} in Table~\ref{tab:LToS}.
	
	\textsf{Infrequent Parameter Update with Small Learning Rate.} Based on the continuity of $\boldsymbol{w}$, a small modification of $\boldsymbol{\phi}$ means a slight modification of local reward functions, and will intuitively result in an equally slight modification of the low-level value functions. This guarantees the low-level policies are highly reusable.
	
	
	\textsf{Unordered Output.} Essentially, the output of high-level policy network is unordered and has a one-to-one match with each neighbor as input. The output of deep neural network, however, is generally ordered and has trouble in varying with the input order. To settle this, we take advantage of DGN which is insensitive of neighbor order as input. Besides, we modify the structure to make the output keep consistency with the neighbor part of input in the relative order.
	
	\section{Hyperparamaters}
	\label{sec:hyper}
	
	As three experimental scenarios are quite different, we may use different hyperparameters. Note that we also tuned the hyparameters for the baselines with grid search. 
	Table~\ref{tab:Hyperparameters} summarizes the hyperparameters of DQN, DGN that also serves as the low-level network of LToS, and IP. We choose the setting of original DGN in \textit{jungle} while the setting of \citet{wei2019colight} in \textit{traffic} for consistency. Table~\ref{tab:LToS} summarizes the hyperparameters of the high-level network of LToS, which are different from the low-level network. Table~\ref{tab:A2C} summarizes the hyperparameters of NeurComm and ConseNet, which adhere to the implementation \citep{chu2020multi}. In addition, for tabular Coco-Q, the step-size parameter is $0.5$, and for IP, the regularizer factor is $0.2$. We adopt soft update for both high-level and low-level networks and use an Ornstein-Uhlenbeck Process (abbreviated as OU) for high-level exploration. 
	
	Both \textit{fixed} LToS and NeurComm exploit static reward shaping, but they adopt different reward shaping schemes which are hard to compare directly. We consider a simple indicator: Self Neighbor Ratio (SNR), the ratio of reward proportion that an agent chooses to keep for itself to that it obtains from a single neighbor. As the rest reward is evenly shared with neighbors in LToS, for each agent $i$, we have $\text{SNR} = \nicefrac{\text{selfishness}}{\text{1-selfishness}} \times (|\mathcal{N}_i|-1)$ for LToS, and $\text{SNR} = \nicefrac{1}{\alpha}$ for NeurComm where $\alpha$ is the spatial discount factor. We adjust the initial selfishness and $\alpha$ to set the SNR of both methods at the same level for fair comparison.

	\begin{table*}[!t]
		\centering
		\setlength{\abovecaptionskip}{3pt}
		\renewcommand{\arraystretch}{1}
		\caption{Hyperparameters for DQN, DGN (also serves as the low-level policy network of LToS), and IP}
		\label{tab:Hyperparameters}
		\footnotesize
		\begin{tabular}{c c c c c}
			\toprule   
			Hyperparamater & \textit{Prisoner} & \textit{Jungle} & \textit{Traffic-$6\times6$} & 
			\textit{Traffic-Shenzhen} \\
			\midrule
			sample size & 10 & 10 & 1,000 & 1,000 \\
			batch size & 10 & 10 & 20 & 20 \\
			buffer capacity & 200,000 & 200,000 & 10,000 & 10,000 \\
			$\epsilon_{start},\epsilon_{decay},\epsilon_{end}$ & 0.8/1/0.8 & 0.6/0.996/0.01 & 0.4/0.9/0.05 & 0.4/0.9/0.05 \\
			\midrule
			initializer & random normal & random normal & random normal & random normal \\
			optimizer & Adam & Adam & RMSProp & RMSProp \\
			learning rate & 1e-3 & 1e-4 & 1e-3 & 1e-3 \\
			$\gamma$ & 0.99 & 0.96 & 0.8 & 0.8 \\
			$\tau$ for soft update & 0.1 & 0.01 & 0.1 & 0.1 \\
			\midrule
			\# MLP units & 32 \& 32 & 512 \& 128 & 32 \& 32 & 32 \& 32 \\
			MLP activation & ReLU & ReLU & ReLU & ReLU \\
			\# encoder MLP layers & 2 & 2 & 2 & 2 \\
			\# attention heads for DGN & 4 & 4 & 1 & 1 \\ 
			\bottomrule  
		\end{tabular}
	\end{table*}
	
	\begin{table*}[!t]
		\centering
		\setlength{\abovecaptionskip}{3pt}
		\renewcommand{\arraystretch}{1}
		\caption{Hyperparameters for the high-level policy network of LToS}
		\label{tab:LToS}
		\footnotesize
		\begin{tabular}{c c c c c}
			\toprule   
			Hyperparamater & \textit{Prisoner} & \textit{Jungle} & \textit{Traffic-$6\times6$} & 
			\textit{Traffic-Shenzhen} \\
			\midrule
            update frequency & 1 step & 100 episodes & 20 episodes & 50 episodes \\
			action interval & 1 step & 1 step & 15 steps & 15 steps \\
			sample size & 2,000 & 5,000 & 1,000 & 1,000 \\
			batch size & 32 & 32 & 20 & 20 \\
			noise for exploration & $\epsilon$ + Gaussian & OU & OU & OU \\
			noise parameter & $\epsilon=0.8$, $\sigma=1$ & $\sigma=0.025\epsilon$ & $\sigma=0.25\epsilon$ & $\sigma=0.25\epsilon$ \\
			\midrule
			initializer & selfishness & selfishness & selfishness & selfishness \\
			initial selfishness & 0.5 & 0.5 & 0.8 & 0.9 \\
			optimizer & SGD & SGD & SGD & SGD \\
			learning rate & 1e-1 & 1e-4 & 1e-3 & 1e-3 \\
			last MLP layer activation & softmax & softmax & softmax & softmax \\
			\bottomrule  
		\end{tabular}
	\end{table*}
	
	\begin{table*}[!t]
		\centering
		\setlength{\abovecaptionskip}{3pt}
		\renewcommand{\arraystretch}{1}
		\caption{Hyperparameters for NeurComm, ConseNet, LIO and QMIX}
		\label{tab:A2C}
		\footnotesize
		\begin{tabular}{c c c c c}
			\toprule   
			Hyperparamater & \textit{Prisoner} & \textit{Jungle} & \textit{Traffic-$6\times6$} & 
			\textit{Traffic-Shenzhen} \\
			\midrule
			initializer & orthogonal & orthogonal & orthogonal & orthogonal \\
			optimizer & RMSProp & RMSProp & RMSProp & RMSProp \\
			learning rate & 5e-3 & 5e-5 & 5e-4 & 5e-4 \\
			\midrule
			\# MLP units & 20 & 512 \& 128 & 16 & 16 \\
			MLP activation & ReLU & ReLU & ReLU & ReLU \\
			\# cell state units & 20 & 512 & 16 & 16 \\
			\# hidden state units & 20 & 512 & 16 & 16 \\
			\midrule 
			RNN type for NeurComm and ConseNet & LSTM & LSTM & LSTM & LSTM \\
			RNN type for QMIX & GRU & GRU & GRU & GRU \\
			hypernetwork layer1 units for QMIX & $2\times20$ & $20\times512$ & $36\times16$ & $36\times16$ \\
			hypernetwork layer2 units for QMIX & $20$ & $512$ & $16$ & $16$ \\
			$\alpha$ for NeurComm & 1 & 0.33 & 0.1 & 0.1 \\
			\midrule 
			$\epsilon_{start},\epsilon_{decay},\epsilon_{end}$ for LIO & 0.8/0.99/0.01 &
			0.6/0.996/0.01 & 0.2/0.9/0.01 & 0.2/0.9/0.01 \\
			$\alpha_{\theta}$ for LIO & 1 & 1e-4 & 1e-4 & 1e-4 \\
			$R_{max}$ for LIO & 2 & 3 & 0.1 & 0.1 \\
			\bottomrule  
		\end{tabular}
	\end{table*}
	
\end{document}